\documentclass[11pt,twoside]{article}
\usepackage{geometry}
\usepackage{fancyhdr}
\fancypagestyle{firststyle}
{
   \fancyhf{}
   \fancyfoot{}
}

\usepackage[utf8]{inputenc} 
\usepackage[T1]{fontenc}    
\usepackage{hyperref}       
\usepackage{url}            
\usepackage{booktabs}       
\usepackage{amsfonts}       
\usepackage{nicefrac}       
\usepackage{microtype}      

\usepackage{todonotes}
\usepackage{wrapfig}

\usepackage{amsthm}
\usepackage{amssymb}
\usepackage{amsmath}
\usepackage{bbm}
\usepackage[title]{appendix}

\usepackage{paralist}
\usepackage[linesnumbered,ruled]{algorithm2e}

\newtheorem{lemma}{Lemma}
\newtheorem{theorem}{Theorem}

\newcommand{\N}{\mathbb{N}}
\newcommand{\R}{\mathbb{R}}

\newcommand{\Order}{\mathcal{O}}
\newcommand{\Expectation}{\mathbb{E}}
\newcommand{\indicator}[1]{\mathbf{1}_{\left\{#1\right\}}}
\newcommand{\fp}[2]{\frac{\partial #1}{\partial #2}}
\newcommand{\one}{\mathbbm{1}}
\DeclareMathOperator{\sign}{sign}

\begin{document}

\title{Dual SVM Training on a Budget}
\author{Sahar Qaadan, Merlin Sch{\"u}ler and Tobias Glasmachers\\Institut f{\"u}r Neuroinformatik, Ruhr-Universit\"at Bochum, Germany\\\texttt{\{sahar.qaadan, merlin.schueler, tobias.glasmachers\}@ini.rub.de}}
\date{}

\maketitle

\begin{abstract}
We present a dual subspace ascent algorithm for support vector machine
training that respects a budget constraint limiting the number of
support vectors. Budget methods are effective for reducing the training
time of kernel SVM while retaining high accuracy. To date, budget training
is available only for primal (SGD-based) solvers. Dual subspace ascent
methods like sequential minimal optimization are attractive for their
good adaptation to the problem structure, their fast convergence rate,
and their practical speed. By incorporating a budget constraint into a
dual algorithm, our method enjoys the best of both worlds. We demonstrate
considerable speed-ups over primal budget training methods.
\end{abstract}

\section{Introduction}

Support Vector Machines (SVMs) introduced by
\cite{cortes-vapnik:support} are popular machine learning methods, in
particular for binary classification. They are supported by
learning-theoretical guarantees \cite{mohri2012foundations}, and they
exhibit excellent generalization performance in many applications in
science and technology
\cite{benhur-et-al:comp-bio,noble:comp-bio,yu-et-al:medicine,son-et-al:medicine,Abe:patt-recog,byun-seowhan:patt-recog,pelossof-et-al:robotics,quinlan-et-al:techniquesfor,joachims:text-cat}.
They belong to the family of kernel methods, applying a linear algorithm
in a feature space defined implicitly by a kernel function.

Training an SVM corresponds to solving a large-scale optimization
problem, which can be cast into a quadratic program (QP). The primal
problem can be solved directly with stochastic gradient descent (SGD)
and accelerated variants \cite{Shalev-et-al:pegasos,glasmachers2016finite},
while the dual QP is solved with subspace ascent, see
\cite{Bottou:supportvector} and references therein.

The computational complexity of each stochastic gradient or coordinate
step is governed by the cost of evaluating the model of a training
point. This cost is proportional to the number of support vectors, which
grows at a linear rate with the data set size
\cite{steinwart2003sparseness}. This limits the applicability of kernel
methods to large-scale data. Efficient algorithms are available for
linear SVMs (SVMs without kernel)
\cite{Shalev-et-al:pegasos,Fan:2008}. Parallelization can yield
considerable speed-ups \cite{wenthundersvm17}, but only by a constant
factor. For non-linear (kernelized) SVMs there exists a wide variety of
approaches for approximate SVM training, many of which aim to leverage
fast linear solvers by approximating the feature space representation of
the data. The approximation can either be fixed (e.g., random Fourier
features) or data-dependent (e.g., Nystr{\"o}m sampling)
\cite{rahimi2008random,yang2012nystrom}.

The budget method imposes an a-priori limit $B \ll n$ on the number of
support vectors \cite{dekel-singer:svm-on-budget}, and hence on the
iteration complexity. In particular with the popular budget maintenance
heuristic of merging support vectors \cite{Wang-et-al:bck}, it goes
beyond the above techniques by adapting the feature space representation
during training. The technique is known as budgeted stochastic gradient
descent (BSGD).

In this context we design the first dual SVM training algorithm with a
budget constraint. The solver aims at the efficiency of dual subspace
ascent as used in LIBSVM, ThunderSVM, and also in LIBLINEAR
\cite{Chang:2011,Fan:2008,wenthundersvm17}, while applying
merging-based budget maintenance as in the BSGD method
\cite{Wang-et-al:bck}. The combination is far from straight-forward,
since continually changing the feature representation also implies
changing the dual QP, which hence becomes a moving target. Nevertheless,
we provide guarantees roughly comparable to those available for BSGD.

In a nutshell, our contributions are:
\begin{compactitem}
\item
	We present the first dual decomposition algorithm operating on a budget,
\item
	we analyze its convergence behavior,
\item
	and we establish empirically its superiority to primal BSGD.
\end{compactitem}

The structure of the paper is as follows: In the next section we
introduce SVMs and existing primal and dual solvers, including BSGD.
Then we present our novel dual budget algorithm and analyze its
asymptotic behavior. We compare our method empirically to BSGD and
validate our theoretical analysis. We close with our conclusions.


\section{Support Vector Machine Training}

A Support Vector Machine is a supervised kernel learning algorithm
\cite{cortes-vapnik:support}. Given labeled training data
$(x_1, y_1), \dots, (x_n, y_n) \in X \times Y$ and a kernel function
$k : X \times X \to \R$ over the input space, the SVM decision function
$f(x) \mapsto \langle w, \phi(x) \rangle$ (we drop the bias, c.f.\
\cite{steinwart2011training}) is defined as the optimal solution
$w^*$ of the (primal) optimization problem
\begin{align}
	\min_{w \in \mathcal{H}} \quad P(w) = \frac{\lambda}{2} \|w\|^2 + \frac{1}{n} \sum_{i=1}^n L\big(y_i, f(x_i)\big), \label{eq:primal}
\end{align}
where $\lambda > 0$ is a regularization parameter, $L$ is a loss
function (usually convex in $w$, turning problem~\eqref{eq:primal} into
a convex problem), and $\phi : X \to \mathcal{H}$ is an only implicitly
defined feature map into the reproducing kernel Hilbert space
$\mathcal{H}$, fulfilling $\langle \phi(x), \phi(x') \rangle = k(x, x')$.
The representer theorem allows to restrict the solution to the form
$w = \sum_{i=1}^n \alpha_i y_i \phi(x_i)$ with coefficient vector
$\alpha \in \R^n$, yielding $f(x) = \sum_{i=1}^n \alpha_i y_i k(x, x_i)$.
Training points $x_i$ with non-zero coefficients $\alpha_i \not= 0$ are
called support vectors.

We focus on the simplest case of binary classification with label space
$Y = \{-1, +1\}$, hinge loss $L(y, f(x)) = \max\{0, 1 - y f(x)\}$, and
classifier $x \mapsto \sign(f(x))$, however, noting that other tasks
like multi-class classification and regression can be tackled in the
exact same framework, with minor changes. For binary classification, the
equivalent dual problem \cite{Bottou:supportvector} reads
\begin{align}
	\max_{\alpha \in [0, C]^n} \quad D(\alpha) = \one^T \alpha - \frac{1}{2} \alpha^T Q \alpha, \label{eq:dual}
\end{align}
which is a box-constrained quadratic program (QP), with
$\one = (1, \dots, 1)^T$ and $C = \frac{1}{\lambda n}$. The matrix $Q$
consists of the entries $Q_{ij} = y_i y_j k(x_i, x_j)$.

\paragraph{Kernel SVM Solvers}

Dual decomposition solvers like LIBSVM \cite{Chang:2011,Bottou:supportvector}
are the method of choice for obtaining a high-precision non-linear
(kernelized) SVM solution. They work by decomposing the dual
problem into a sequence of smaller problems of size $\Order(1)$, and
solving the overall problem in a subspace ascent manner. For
problem~\eqref{eq:dual} this can amount to coordinate ascent (CA).
Keeping track of the dual gradient
$\nabla_\alpha D(\alpha) = \one - Q \alpha$ allows for the application
of elaborate heuristics for deciding which coordinate to optimize next,
based on the violation of the Karush-Kuhn-Tucker conditions or even
taking second order information into account. Provided that coordinate
$i$ is to be optimized in the current iteration, the sub-problem
restricted to $\alpha_i$ is a one-dimensional QP, which is solved
optimally by the truncated Newton step
\begin{align}
	\alpha_i \leftarrow \left[ \alpha_i + \frac{1 - Q_i \alpha}{Q_{ii}} \right]_0^C, \label{eq:update-dual}
\end{align}
where $Q_i$ is the $i$-th row of $Q$ and $[x]_0^C = \max\big\{0, \min\{C, x\}\big\}$
denotes truncation to the box constraints. The method enjoys locally
linear convergence \cite{lin2001convergence}, polynomial worst-case
complexity \cite{list2005general}, and fast convergence in practice.

In principle the primal problem~\eqref{eq:primal} can be solved directly,
e.g., with SGD, which is at the core of the kernelized Pegasos algorithm
\cite{Shalev-et-al:pegasos}. Replacing the average loss (empirical risk)
in equation~\eqref{eq:primal} with the loss $L(y_i, f(x_i))$ on a single
training point selected uniformly at random provides an unbiased
estimate. Following its (stochastic) sub-gradient with learning rate
$1/(\lambda t) = (n C)/t$ in iteration $t$ yields the update
\begin{align}
	\alpha \leftarrow \alpha - \frac{\alpha}{t} + \indicator{y_i f(x_i) < 1} \frac{n C}{t} e_i, \label{eq:update-primal}
\end{align}
where $e_i$ is the $i$-th unit vector and $\indicator{E}$ is the
indicator function of the event $E$. Despite fast initial progress, the
procedure can take a long time to produce accurate results, since SGD
suffers from the non-smooth hinge loss, resulting in slow convergence.

In both algorithms, the iteration complexity is governed by the
computation of $f(x)$ (or equivalently, by the update of the dual
gradient), which is linear in the number of non-zero coefficients
$\alpha_i$. This is a limiting factor when working with large-scale
data, since the number of support vectors is usually linear in the data
set size $n$ \cite{steinwart2003sparseness}.

\paragraph{Linear SVM Solvers}

Specialized solvers for linear SVMs with $X = \R^d$ and $\phi$ chosen
as the identity mapping exploit the fact that the weight vector
$w \in \R^d$ can be represented directly. This lowers the iteration
complexity from $\Order(n)$ to $\Order(d)$ (or the number of non-zero
features in $x_i$), which often results in significant savings
\cite{joachims2006training,Shalev-et-al:pegasos}.
This works even for dual CA by keeping track of the direct
representation $w$ and the (redundant) coefficients $\alpha$, however,
at the price that the algorithm cannot keep track of the dual gradient
any more, which would be an $\Order(n)$ operation. Therefore the
LIBLINEAR solver resorts to uniform coordinate selection
\cite{Fan:2008}, which amounts to stochastic coordinate ascent (SCA)
\cite{nesterov2012efficiency}.

Linear SVMs shine on application domains like text mining, with sparse
data embedded in high-dimensional input spaces. In general, for moderate
data dimension~$d \ll n$, separation of the data with a linear model is
a limiting factor that can result in severe under-fitting.

\paragraph{SVMs on a Budget}

Lowering the iteration complexity is also the motivation for introducing
an upper bound or budget $B \ll n$ on the number of support vectors.
The budget $B$ is exposed to the user as a hyperparameter of the method.
The proceeding amounts to approximating $w$ with a vector $\tilde w$
from the non-trivial fiber bundle
\begin{align*}
	W_B = \left\{ \left. \sum_{j=1}^B \beta_j \phi(\tilde x_j) \,\right|\, \beta_1, \dots, \beta_B \in \R; \enspace \tilde x_1, \dots, \tilde x_B \in \R^d \right\} \subset \mathcal{H}.
\end{align*}
Critically, $W_B$ is in general non-convex, and so are optimization
problems over this set.

Each SGD step (eq.~\eqref{eq:update-primal}) adds at most one new
support vector to the model. If the number of support vectors exceeds
$B$ after such a step, then the budgeted stochastic gradient descent
(BSGD) method applies a budget maintenance heuristic to remove one
support vector. Merging of two support vectors has proven to be a good
compromise between the induced error and the resulting computational
effort \cite{Wang-et-al:bck}. It amounts to replacing
$\beta_i \phi(\tilde x_i) + \beta_j \phi(\tilde x_j)$ (with carefully
chosen indices $i$ and $j$) with a single term $\beta' \phi(\tilde x')$,
aiming to minimize the ``weight degradation'' error
$\|\beta_i \phi(\tilde x_i) + \beta_j \phi(\tilde x_j) - \beta' \phi(\tilde x')\|^2$.
For the widely used Gaussian kernel $k(x, x') = \exp(-\gamma \|x-x'\|^2)$
the optimal $\tilde x'$ lies on the line spanned by $\tilde x_i$ and
$\tilde x_j$, and it is a convex combination if merging is restricted to
points of the same class. The coefficient $h$ of the convex combination
$\tilde x' = (1-h) \tilde x_i + h \tilde x_j$ is found
with golden section search, and the optimal coefficient $\beta'$ is
obtained in closed form.

In effect, merging allows BSGD to move support vectors around in the
input space. This is well justified since restricted to $W_B$ the
representer theorem does not hold. \cite{Wang-et-al:bck} show that
asymptotically the performance is governed by the approximation error
implied by $w^* \not\in W_B$ (see their Theorem~1).

BSGD aims to achieve the best of two worlds, namely a reasonable
compromise between statistical and computational demands: fast training
is achieved through a bounded computational cost per iteration, and the
application of a kernel keeps the model sufficiently flexible. This
requires that $B \ll n$ basis functions are sufficient to represent a
model $\tilde w$ that is sufficiently close to the optimal model $w^*$.
This assumption is very reasonable, in particular for large~$n$.

\section{Dual Coordinate Ascent with Budget Constraint}

In this section we present our novel approximate SVM training algorithm.
At its core it is a dual decomposition algorithm, modified to respect a
budget constraint. It is designed such that the iteration complexity is
limited to $\Order(B)$ operations, and is hence independent of the data
set size~$n$. Our solver combines components from decomposition methods
\cite{Osuna97animproved}, dual linear SVM solvers \cite{Fan:2008}, and
BSGD \cite{Wang-et-al:bck} into a new algorithm. Like BSGD, we aim to
achieve the best of two worlds: a-priori limited iteration complexity
with a budget approach, combined with fast convergence of a dual
decomposition solver. Both aspects speed-up the training process, and
hence allow to scale SVM training to larger problems.

Introducing a budget into a standard decomposition algorithm as
implemented in LIBSVM \cite{Chang:2011} turns out to be non-trivial.
Working with a budget is rather straightforward on the primal
problem~\eqref{eq:primal}. The optimization problem is unconstrained,
allowing BSGD to replace $w$ represented by $\alpha$ transparently with
$\tilde w$ represented by coefficients $\beta_j$ and flexible basis
points $\tilde x_j$. This is not possible for the dual
problem~\eqref{eq:dual} with constraints formulated directly in terms
of~$\alpha$.

This difficulty is solved by \cite{Fan:2008} for the linear SVM
training problem by keeping track of $w$ and $\alpha$. We follow the
same approach, however, in our case the correspondence between $w$
represented by $\alpha$ and $\tilde w$ represented by $\beta_j$ and
$\tilde x_j$ is only approximate. This is unavoidable by the very nature
of the problem. Luckily, this does not impose major additional
complications.

\begin{algorithm}
	{\textbf{Input:} training data $(x_1, y_1), \dots, (x_n, y_n)$, $k : X \times X \to \R$, $C>0$, $B \in \N$}

	{$\alpha \leftarrow 0$, $M \leftarrow \emptyset$}

	\While{not happy}
	{
		{select index $i \in \{1, \dots, n\}$ uniformly at random}

		{$\tilde f(x_i) = \sum_{(\beta, \tilde x) \in M} \beta k(x_i, \tilde x)$ \label{line:fx}}

		{$\delta = \left[ \alpha_i + \big(1 - y_i \tilde f(x_i)\big) / Q_{ii} \right]_0^C - \alpha_i$}

		\If{$\delta \not= 0$}
		{
			{$\alpha_i \leftarrow \alpha_i + \delta$}

			{$M \leftarrow M \cup \{(\delta, x_i)\}$}

			\If{$|M| > B$}
			{
				{trigger budget maintenance}
			}
		}
	}

	\caption{
		Budgeted Stochastic Coordinate Ascent (BSCA) Algorithm
		\label{algo:BSCA}
	}
\end{algorithm}

The pseudo-code of our Budgeted Stochastic Coordinate Ascent (BSCA)
approach is detailed in algorithm~\ref{algo:BSCA}. It represents the
approximate model $\tilde w$ as a set $M$ containing tuples
$(\beta, \tilde x)$. Critically, in line~\ref{line:fx} the approximate
model $\tilde w$ is used to compute
$\tilde f(x_i) = \langle \tilde w, x_i \rangle$, so the complexity of
this step is $\Order(B)$. This is in contrast to the computation of
$f(x_i) = \langle w, x_i \rangle$, with effort is linear in $n$. At the
target iteration cost of $\Order(B)$ it is not possible to keep track of
the dual gradient, simply because it consists of $n$ entries that would
need updating with a dense matrix row $Q_i$.
Consequently, and in line with \cite{Fan:2008}, we resort to uniform
variable selection in an SCA scheme, and the role of the coefficients
$\alpha$ is reduced to keeping track of the constraints.

For the budget maintenance procedure, the same options are available
as in BSGD. It is usually implemented as merging of two support vectors,
reducing a model from size $|M| = B+1$ back to size $|M| = B$. It is
understood that also the complexity of the budget maintenance procedure
should be bounded by $\Order(B)$ operations. Furthermore, for the
overall algorithm to work properly, it is important to maintain the
approximate relation $\tilde w \approx w$. For reasonable settings of
the budget $B$, this is achieved by non-trivial budget maintenance
procedures like merging and projection \cite{Wang-et-al:bck}.

We leave the stopping criterion for the algorithm open. A stopping
criterion akin to \cite{Fan:2008} based on thresholding KKT violations
is not viable, as shown by the subsequent analysis. We therefore run the
algorithm for a fixed number of iterations (or epochs), as it is common
for BSGD.

\section{Analysis of BSCA}

BSCA is an approximate dual training scheme. Therefore two questions of
major interest are how quickly it approaches $w^*$, and how close it
gets.

To simplify matters somewhat, we make the assumption that the matrix $Q$
is strictly positive definite. This ensures that the optimal coefficient
vector $\alpha^*$ corresponding to $w^*$ is unique. For a given weight
vector $w = \sum_{i=1}^n \alpha_i y_i \phi(x_i)$, we write $\alpha(w)$
when referring to the corresponding coefficients, which are also unique.

Let $w^{(t)}$ and $\alpha^{(t)} = \alpha(w^{(t)})$, $t \in \N$, denote
the sequence of solutions generated by an iterative algorithm, using the
labeled training point $(x_{i^{(t)}}, y_{i^{(t)}})$ for its update in
iteration $t$. The indices $i^{(t)} \in \{1, \dots, n\}$ are drawn
i.i.d.\ from the uniform distribution.

\paragraph{Optimization Progress of BSCA}

We start by computing the single-iteration progress.
\begin{lemma} \label{lemma:progress}
The change $D(\alpha^{(t)}) - D(\alpha^{(t-1)})$ of the dual objective
function in iteration $t$ operating on the coordinate index
$i = i^{(t)} \in \{1, \dots, n\}$ equals
\begin{align*}
	J \left( \alpha^{(t-1)}, i, \alpha^{(t)}_i-\alpha^{(t-1)}_i \right)
		:= \frac{Q_{ii}}{2} \left(
			  \left[\frac{1 - Q_i \alpha^{(t-1)}}{Q_{ii}}\right]^2
			- \left[\Big(\alpha^{(t)}_i - \alpha^{(t-1)}_i\Big) - \frac{1 - Q_i \alpha^{(t-1)}}{Q_{ii}}\right]^2
		\right).
\end{align*}
\end{lemma}
\begin{proof}
Consider the function
$s(\delta) = D(\alpha^{(t-1)} + \delta e_i)$. 
It is quadratic with second derivative $-Q_{ii} < 0$ and with its
maximum at $\delta^* = (1 - Q_i \alpha^{(t-1)}) / Q_{ii}$. Represented
by its second order Taylor series around $\delta^*$ it reads
$s(\delta) = s(\delta^*) - \frac{Q_{ii}}{2} (\delta - \delta^*)^2$.
This immediately yields the result.
\end{proof}

The lemma is in line with the optimality of the update~\eqref{eq:update-dual}.
Based thereon we define the relative approximation error
\begin{align*}
	E(w, \tilde w) := 1 - \max_{i \in \{1, \dots, n\}} \left\{
			  \frac{J\left( \alpha(w), i, \left[ \alpha_i(w) + \frac{1 - y_i \langle \tilde w, \phi(x_i) \rangle}{Q_{ii}} \right]_0^C - \alpha_i(w) \right)}
			       {J\left( \alpha(w), i, \left[ \alpha_i(w) + \frac{1 - y_i \langle        w, \phi(x_i) \rangle}{Q_{ii}} \right]_0^C - \alpha_i(w) \right)}
			\right\}.
\end{align*}
The margin calculation in the numerator is based on $\tilde w$, while
it is based on $w$ in the denominator.
Hence $E(w, \tilde w)$ captures the effect of using $\tilde w$ instead
of $w$ in BSCA. Informally, we interpret it as a dual quantity related
to the weight degradation error $\|\tilde w - w\|^2$. The relative
approximation error is non-negative, continuous (and piecewise linear)
in $\tilde w$ (for fixed $w$), and it fulfills
$\tilde w = w \Rightarrow E(w, \tilde w) = 0$. The following theorem
bounds the suboptimality of BSCA, and it captures the intuition that the
relative approximation error poses a principled limit on the achievable
solution precision.

\begin{theorem} \label{theorem:progress}
The sequence $\alpha^{(t)}$ produced by BSCA fulfills
\begin{align*}
	D(\alpha^*) - \Expectation \Big[ D(\alpha^{(t)}) \Big] \leq \left( D(\alpha^*) + \frac{n C^2}{2} \right) \cdot \prod_{\tau=1}^t \left(1 - \frac{2\kappa \big(1-E(w^{(\tau-1)}, \tilde w^{(\tau-1)})\big)}{(1+\kappa) n}\right),
\end{align*}
where $\kappa$ is the smallest eigenvalue of $Q$.
\end{theorem}
\begin{proof}
Theorem~5 by \cite{nesterov2012efficiency} applied to our setting
ensures linear convergence
$$ \Expectation[D(\alpha^*) - D(\alpha^{(t)})] \leq \left( D(\alpha^*) + \frac{n C^2}{2} \right) \cdot \left(1-\frac{2\kappa}{(1+\kappa) n}\right)^{t}, $$
and in fact the proof establishes a linear decay of the expected
suboptimality by the factor $1 - \frac{2\kappa}{(1+\kappa) n}$ in each
single iteration. The improvement is reduced by a factor of at most
$1 - E(w, \tilde w)$, by lemma~\ref{lemma:progress} and the definition
of the relative approximation error.
\end{proof}

We conclude from Theorem~\ref{theorem:progress} that the behavior of
BSCA can be divided into and early and a late phase. For fixed weight
degradation, the relative approximation error is small as long as the
progress is sufficiently large, which is the case in early iterations.
Then the algorithm is nearly unaffected by the budget constraint, and
multiplicative progress at a fixed rate is achieved. Progress gradually
decays when approaching the optimum, which increases the relative
approximation error, until BSCA stalls. In fact, the theorem does not
witness further progress for $E(w, \tilde w) \geq 1$. Due to
$w^* \not\in W_B$, the KKT violations do not decay to zero, and the
algorithm approaches a limit distribution.%
\footnote{BSCA does not converge to a unique point. It does not become
  clear from the analysis provided by \cite{Wang-et-al:bck} whether
  this is also the case for BSGD, or whether the decaying learning rate
  allows BSGD to converge to a local minimum.}
The precision to which the optimal SVM solution can be approximated is
hence limited by the relative approximation error, or indirectly, by the
weight degradation.

\paragraph{Budget Maintenance Rate}

The rate at which budget maintenance is triggered can play a role, in
particular if the procedure consumes a considerable share of the overall
runtime. In the following we highlight a difference between BSGD and
BSCA. For an algorithm $A$ let
\begin{align*}
	p_A = \lim_{T \to \infty} \Expectation\left[ \frac{1}{T} \cdot \Big|\Big\{ t \in \{1, \dots, T\} \,\Big|\, y_{i^{(t)}} \langle w^{(t-1)}, \phi(x_{i^{(t)}}) \rangle < 1 \Big\}\Big| \right]
\end{align*}
denote the expected fraction of optimization steps in which the target
margin is violated, in the limit $t \to \infty$ (if the limit exists).
The following lemma establishes the fraction for primal SGD
(eq.~\eqref{eq:update-primal}) and dual SCA (eq.~\eqref{eq:update-dual}),
both without budget.

\begin{lemma} \label{lemma:rate}
Under the conditions
(i) $\alpha^*_i \in \{0, C\} \Rightarrow \fp{D(\alpha^*)}{\alpha_i} \not= 0$
and
(ii) $\fp{D(\alpha^{(t)})}{\alpha_{i_t}} \not= 0$
(excluding only a zero-set of cases) it holds
$p_\text{SGD} = \frac{1}{n}\sum_{i=1}^n \frac{\alpha^*_i}{C}$ and
$p_\text{SCA} = \frac{1}{n} |\{i \,|\, 0 < \alpha^*_i < C\}|$.
\end{lemma}
\begin{proof}
In the update equation \eqref{eq:update-primal}, due to
$\alpha^{(t)} \to \alpha^*$ and $\sum_{t=1}^\infty \frac{1}{t} = \infty$,
the subtraction of $\alpha^{(t-1)}_i$ and the addition of $nC$ with
learning rate $\frac{1}{t}$ must cancel out in the limit $t \to \infty$,
in expectation. Formally speaking, we obtain
$$ \lim_{T \to \infty} \Expectation\left[ \frac{1}{T} \sum_{t=1}^T \indicator{i^{(t)} = i} \indicator{y_{i^{(t)}} \langle w^{(t-1)}, \phi(x_{i^{(t)}}) \rangle < 1} n C - \alpha^{(t-1)}_i \right] = 0 \quad \forall i \in \{1, \dots, n\}, $$
and hence
$\lim_{T \to \infty} \Expectation\left[ \frac{1}{T} \sum_{t=1}^T \indicator{i^{(t)} = i} \indicator{y_{i^{(t)}} \langle w^{(t-1)}, \phi(x_{i^{(t)}}) \rangle < 1} \right] = \frac{\alpha^*_i}{n C}$.
Summation over $i$ completes the proof of the result for SGD.

In the dual algorithm, with condition (i) and the same argument as in
\cite{lin2001convergence} there exists an iteration $t_0$ so that for
$t > t_0$ all variables fulfilling $\alpha^*_i \in \{0, C\}$ remain
fixed: $\alpha^{(t)}_i = \alpha^{(t_0)}_i$, while all other variables
remain free: $0 < \alpha^{(t)}_i < C$. Assumption (ii) ensures that all
steps on free variables are non-zero and hence contribute $1/n$ to
$p_\text{SCA}$ in expectation, which yields
$p_\text{SCA} = \frac{1}{n} |\{i \,|\, 0 < \alpha^*_i < C\}|$.
\end{proof}

A point $(x_{i^{(t)}}, y_{i^{(t)}})$ that violates the target margin of
one is added as a new support vector in BSGD as well as in BSCA. After
the first $B$ such steps, all further additions trigger budget
maintenance. Hence Lemma~\ref{lemma:rate} gives an asymptotic indication
of the number of budget maintenance events, provided
$\tilde w \approx w$, i.e., if the budget is not too small. The
different rates for primal and dual algorithm underline the quite
different optimization behavior of the two algorithms: while (B)SGD
keeps making non-trivial steps on all training points corresponding to
$\alpha^*_i > 0$ (support vectors w.r.t.\ $w^*$), after a while the dual
algorithm operates only on the free variables $0 < \alpha^*_i < C$.

\section{Experiments}

In this section we compare our dual BSCA algorithm to the primal BSGD
method on the binary classification problems ADULT, COD-RNA, COVERTYPE,
IJCNN, and SUSY, covering a range of different sizes. The
regularization parameter $C = \frac{1}{n \cdot \lambda}$ and the kernel
parameter $\gamma$ were tuned with grid search and cross-validation, see
table~\ref{table:dataSetsProperties}.
Due to space constraints, we present selected, representative results
in the paper. Additional results including runs on a smaller budget are
found in the appendix.

\begin{table}[h]
\begin{center}
	\caption{
		\label{table:dataSetsProperties}
		Data sets used in this study, hyperparameter settings, test accuracy and training time of the full SVM model (using LIBSVM).
		On the SUSY problem, we aborted the extremely long LIBSVM run. The corresponding accuracy is the cross-validation result on a subset.
	}
	\begin{tabular}{lrrllrr}
	\hline
	data set & size & features & $C$ & $\gamma$ & accuracy & training time \\
	\hline
	SUSY & 4,500,000 & 18 & $2^5$ & $2^{-7}$ & $79.79\%$ & $>180h$ \\
	COVTYPE & 581,012 & 54 & $2^7$ & $2^{-3}$ & $75.88\%$ & $45395.729s$\\
	COD-RNA & 59,535 & 8 & $2^5$ & $2^{-3}$ & $96.33\%$ & $53.951s$\\
	IJCNN & 49,990 & 22 & $2^5$ & $2^1$ & $98.77\%$ & $46.914s$\\
	ADULT & 32,561 & 123 & $2^5$ & $2^{-7}$ & $84.82\%$ & $97.152s$ \\
	\hline
	\end{tabular}
\end{center}
\end{table}

\paragraph{Optimization Performance}
BSCA and BSGD are optimization algorithms. Hence it is natural to
compare them in terms of primal and dual objective function, see
equations \eqref{eq:primal} and \eqref{eq:dual}. Since the solvers
optimize different functions, we monitor both. However, we consider the
primal as being of primary interest since its minimization is the goal
of SVM training, by definition. Convergence plots are displayed in
figure~\ref{figure:primaldualperf}. Overall, the dual BSCA solver
clearly outperforms the primal BSGD method across all problems.
While the dual objective function values are found to be smooth and
monotonic in all cases, this is not the case for the primal.

BSCA generally oscillates less and stabilizes faster (with the exception
of the ADULT problem), while BSGD remains somewhat unstable and is hence
at risk of delivering low-quality solutions for a much longer time when
it happens to stop at one of the peaks.

\begin{figure}[h]
\begin{center}

	\includegraphics[width=0.32\columnwidth]{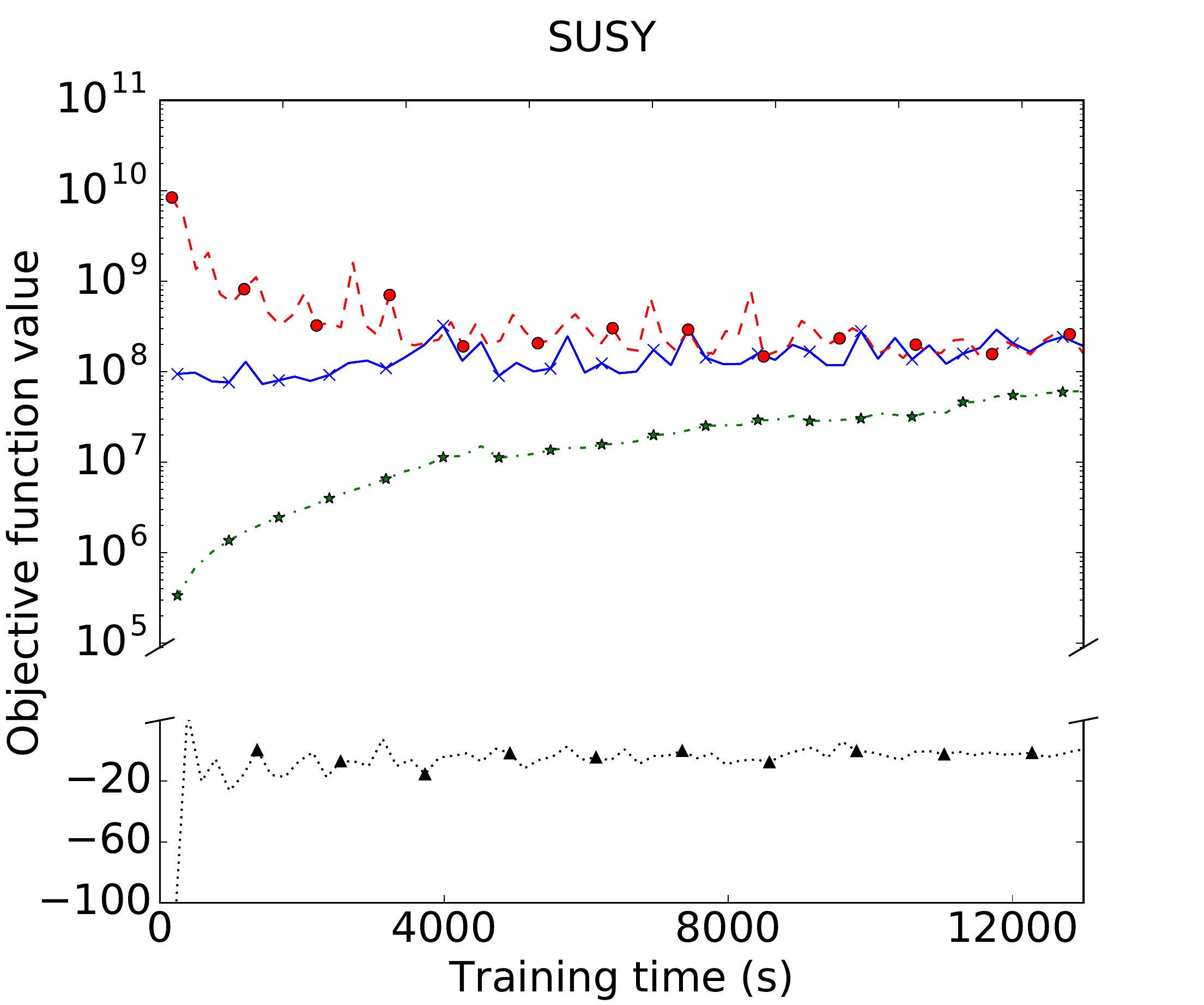}
	\includegraphics[width=0.32\columnwidth]{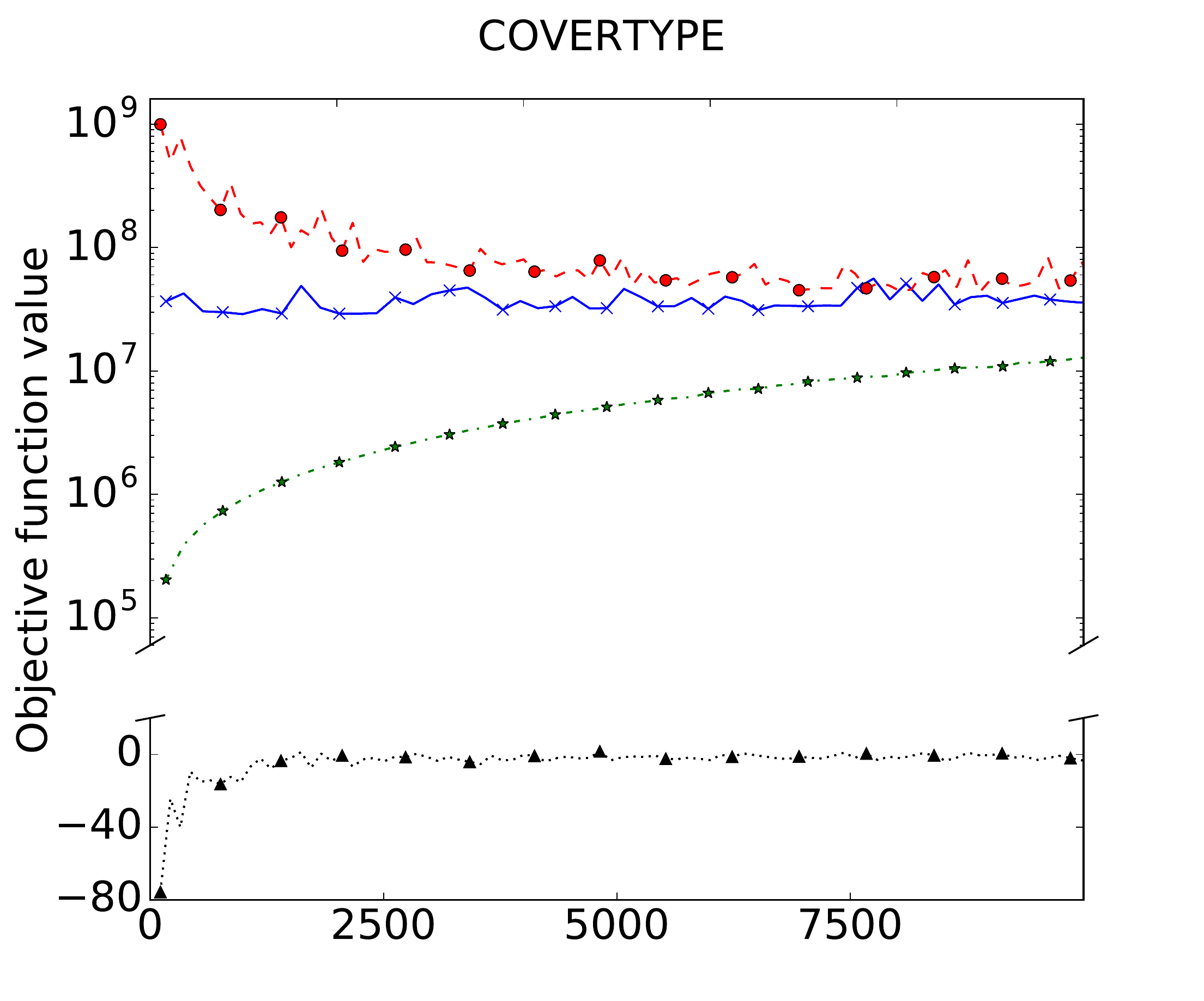}
	\includegraphics[width=0.32\columnwidth]{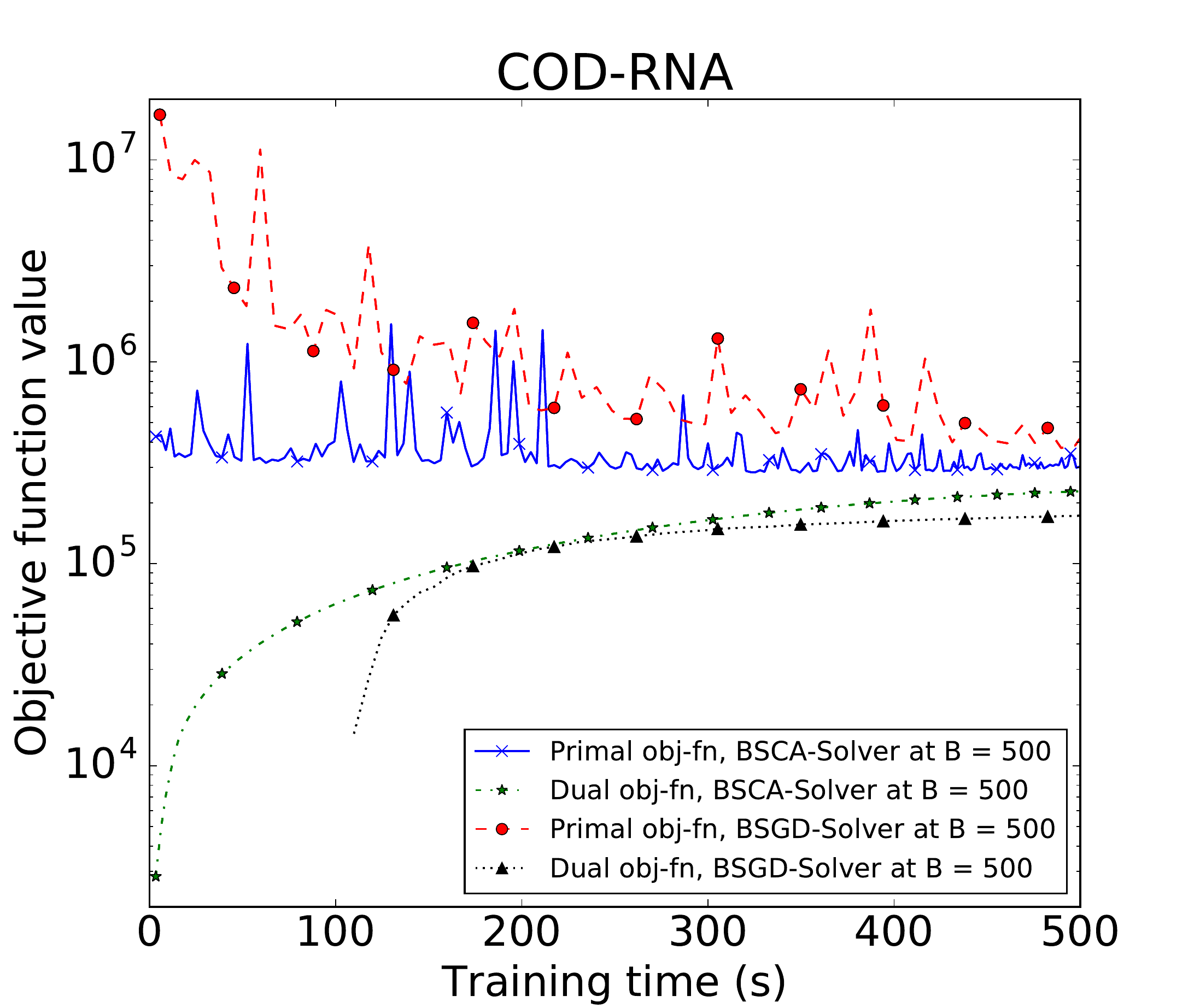}
	\\
	\includegraphics[width=0.32\columnwidth]{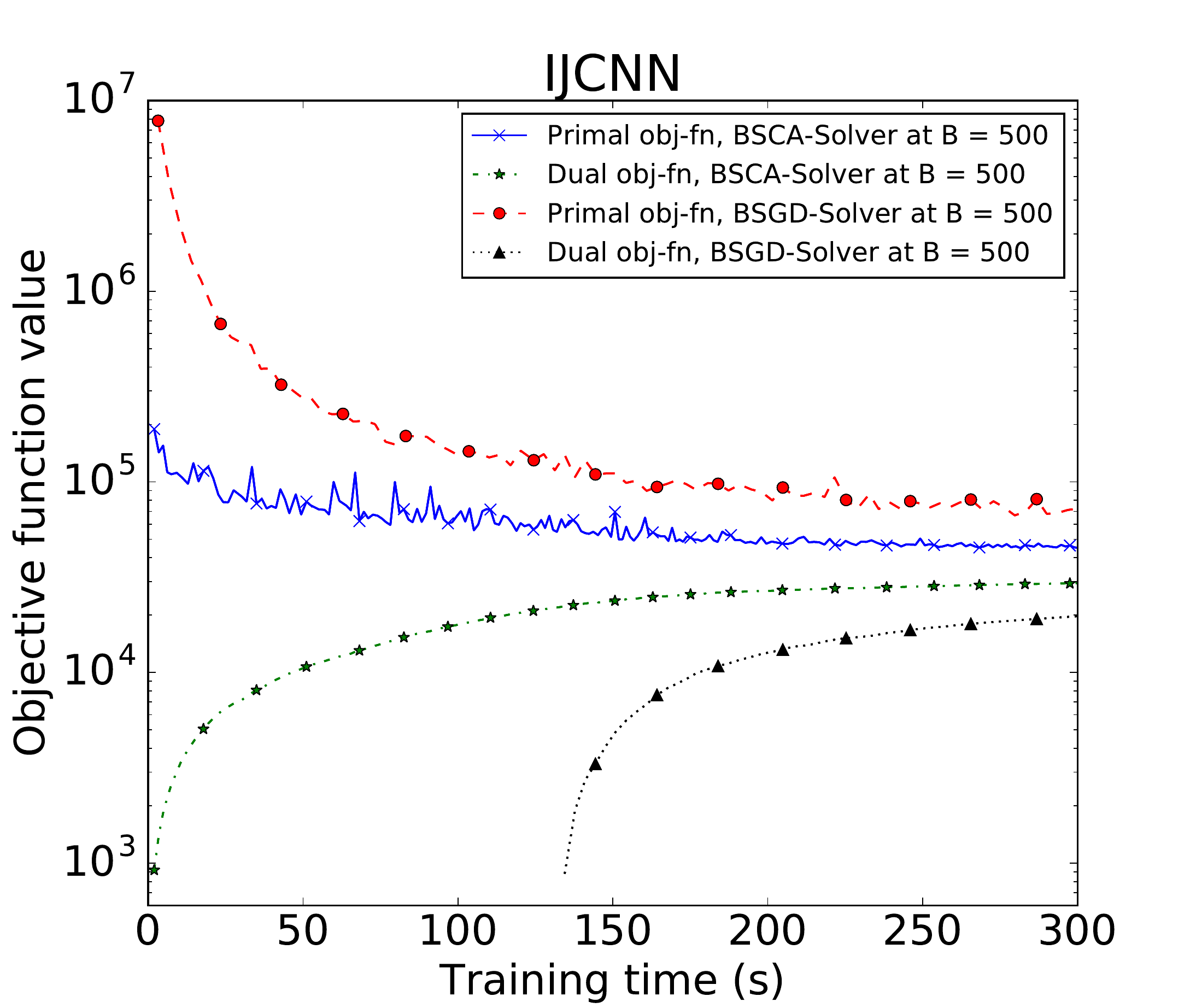}
	\includegraphics[width=0.32\columnwidth]{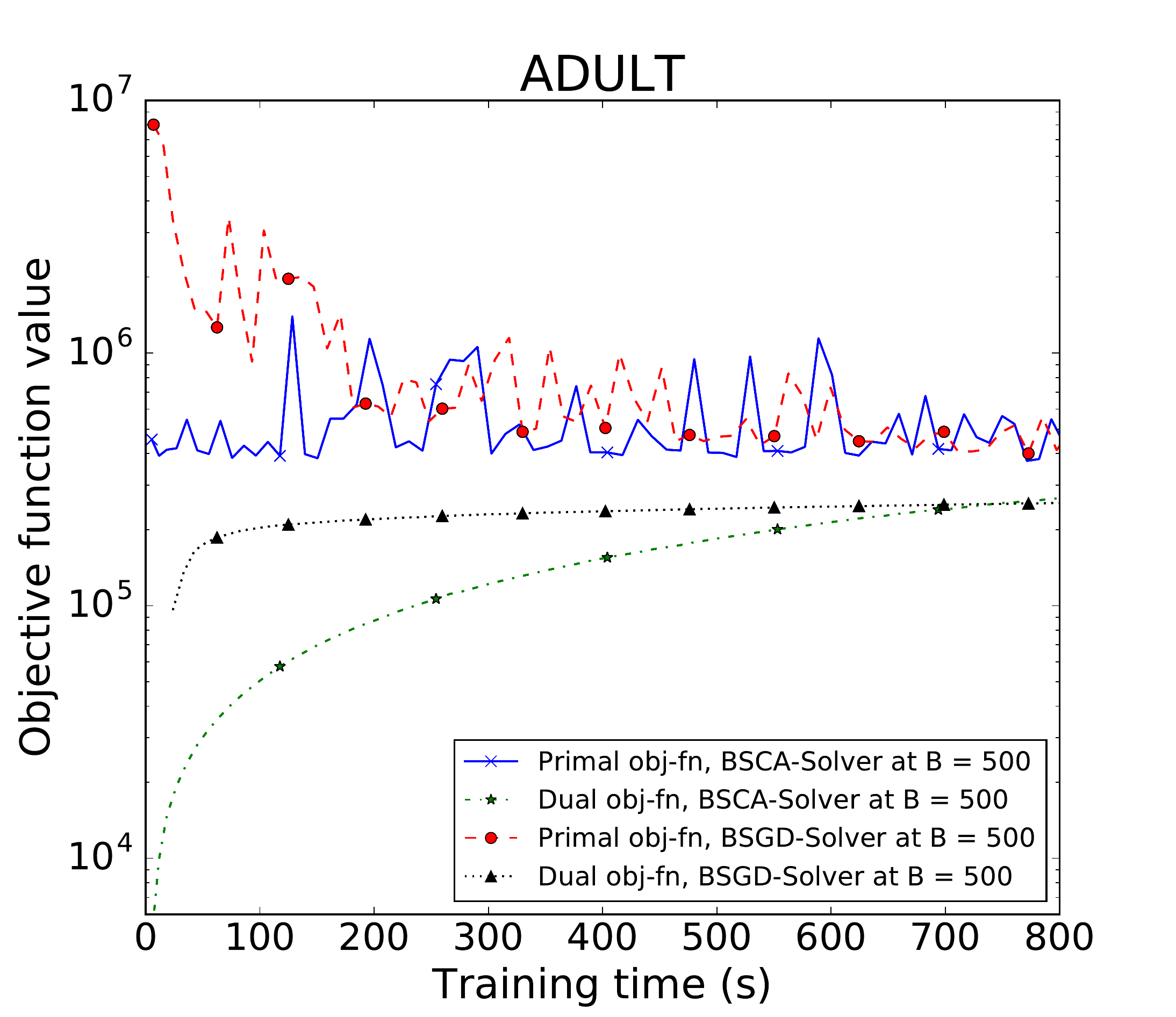}
\end{center}
\caption{
	Primal and dual objective function curves of BSGD and BSCA solvers with a budget of $B=500$.
	We use a mixed linear and logarithmic scale where the dual objective
	stays negative, which happens for BSGD on two problems.
	\label{figure:objfn}
}
\end{figure}

\paragraph{Learning Performance}
\begin{figure}[h]
\begin{center}
	\includegraphics[width=0.32\columnwidth]{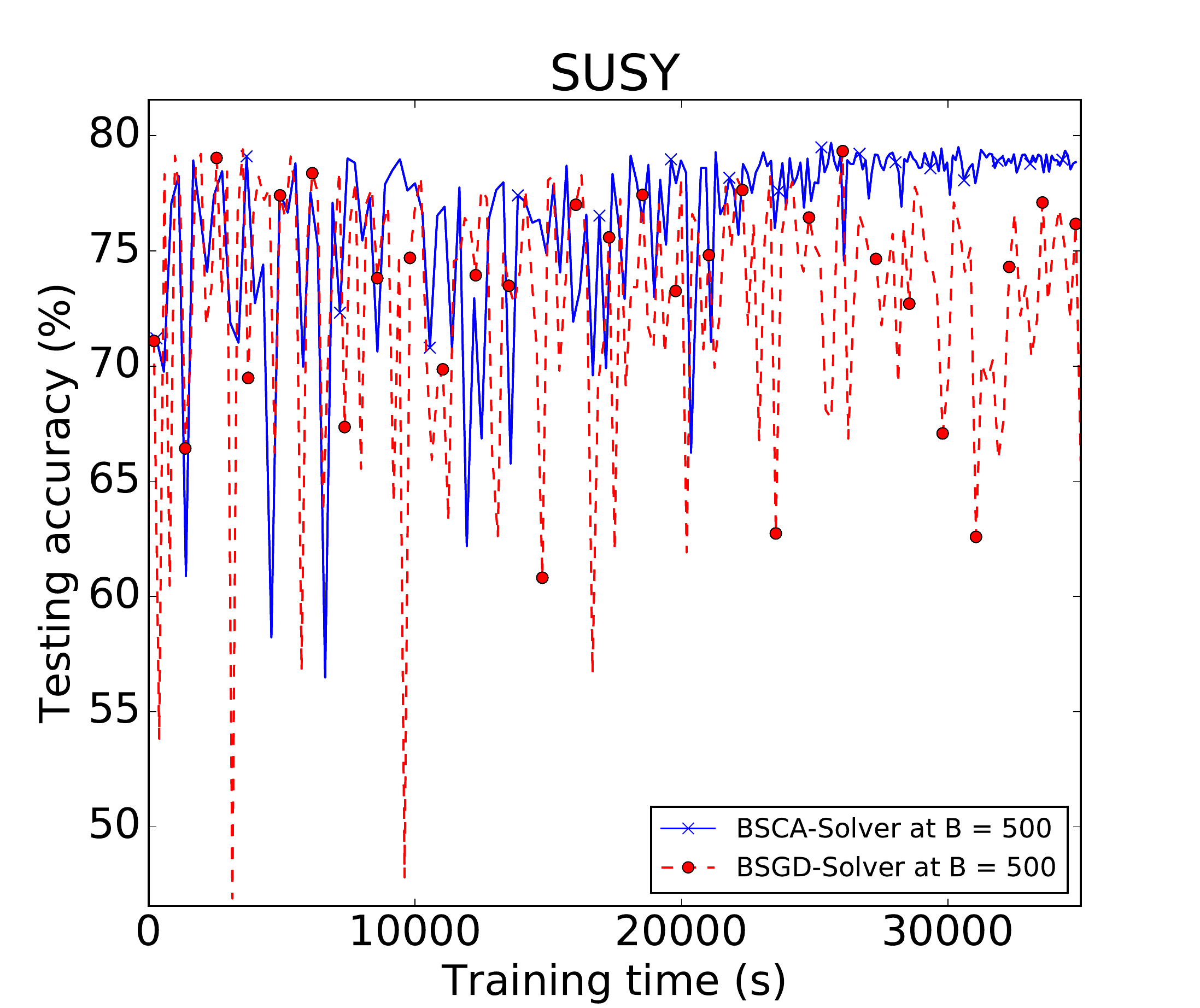} 
	\includegraphics[width=0.32\columnwidth]{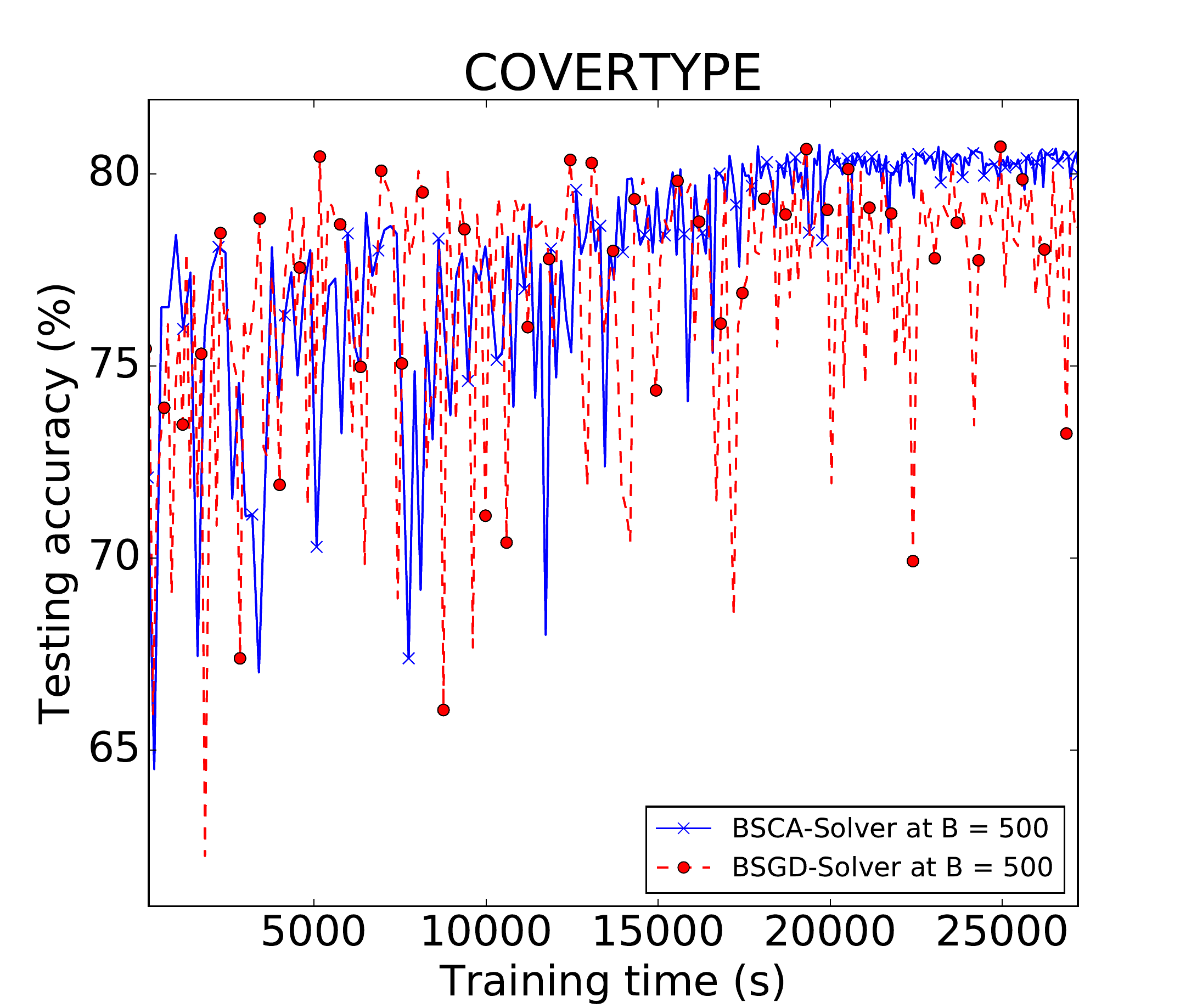}
	\includegraphics[width=0.32\columnwidth]{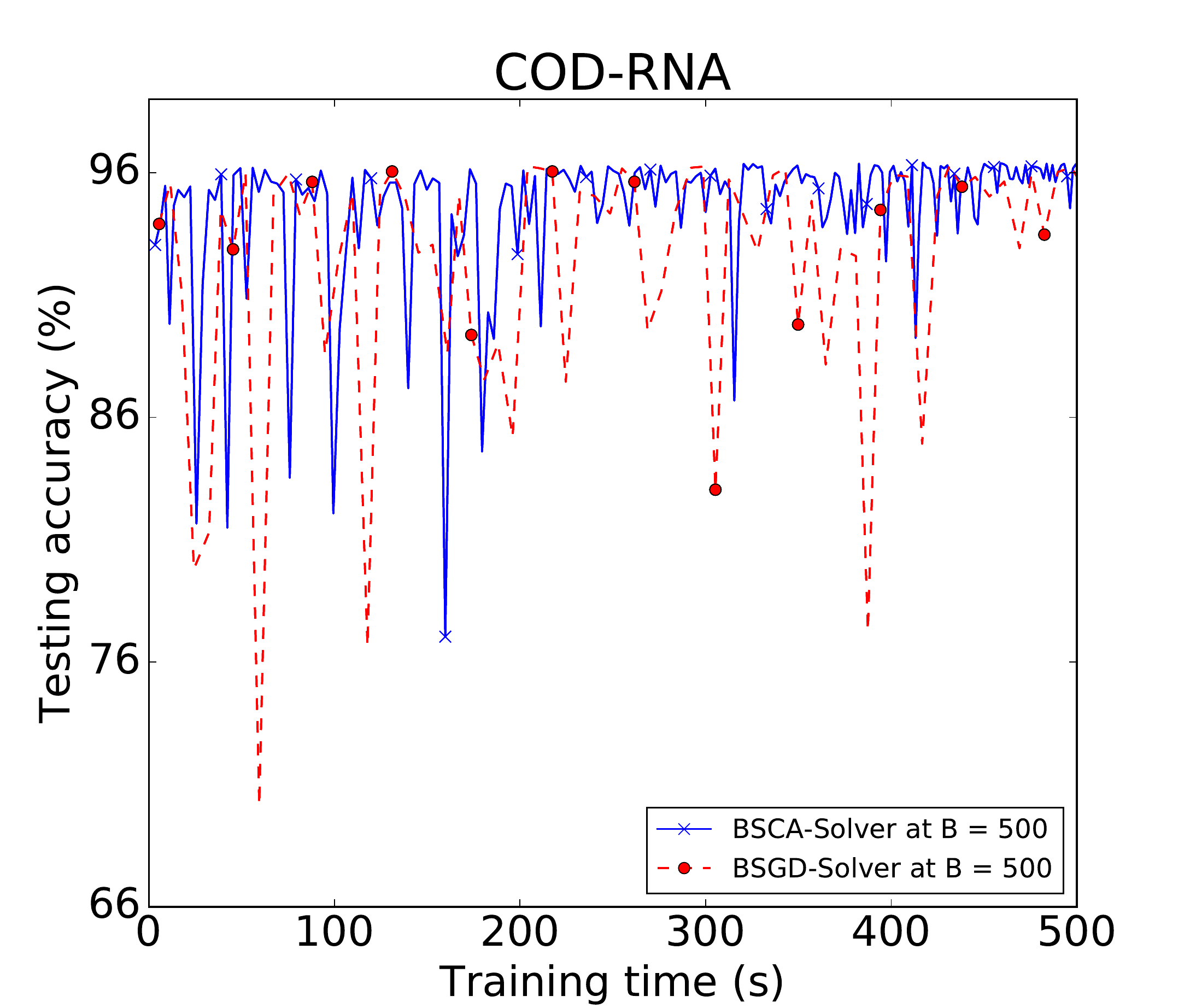}
	\\
	\includegraphics[width=0.32\columnwidth]{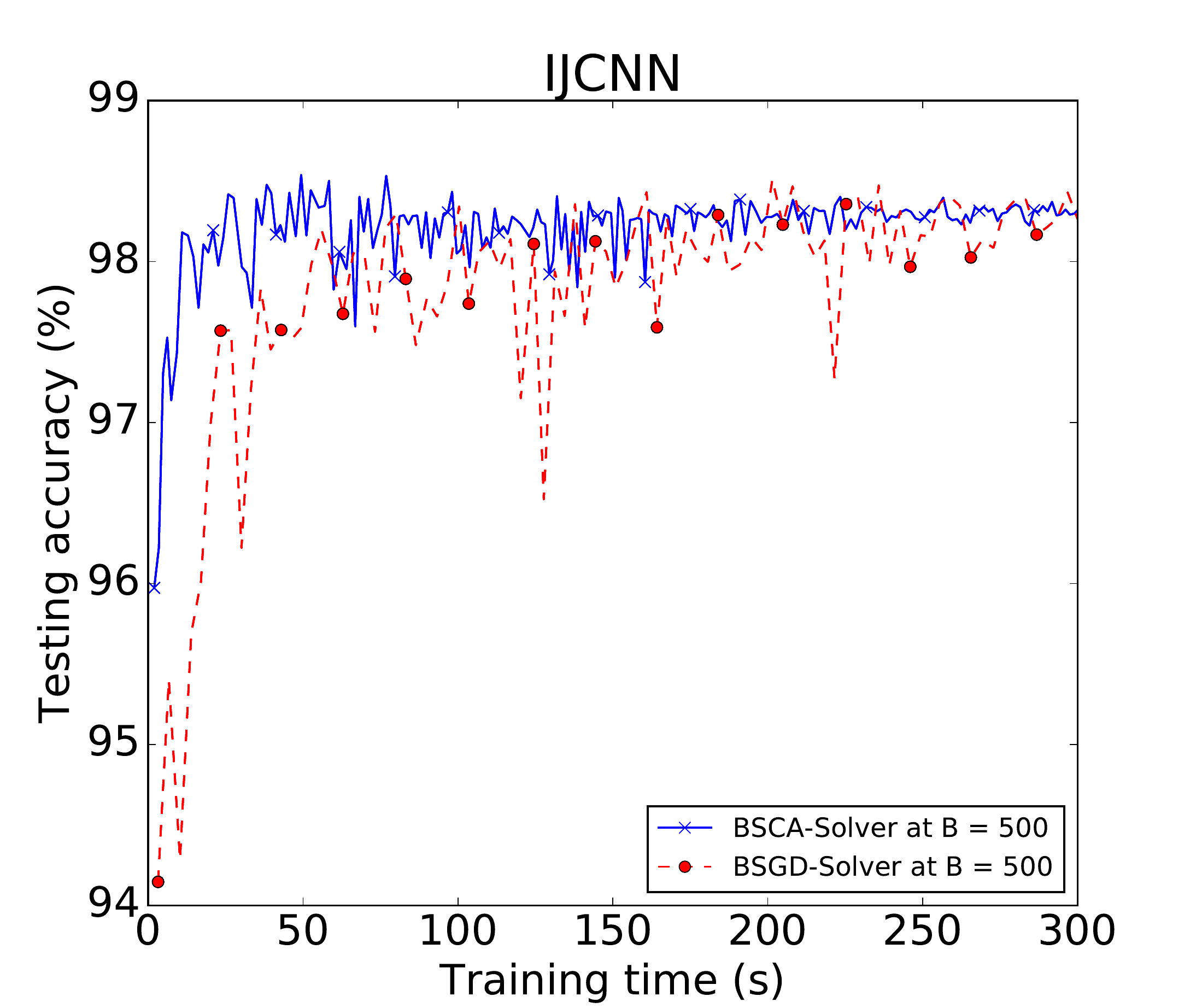}
	\includegraphics[width=0.32\columnwidth]{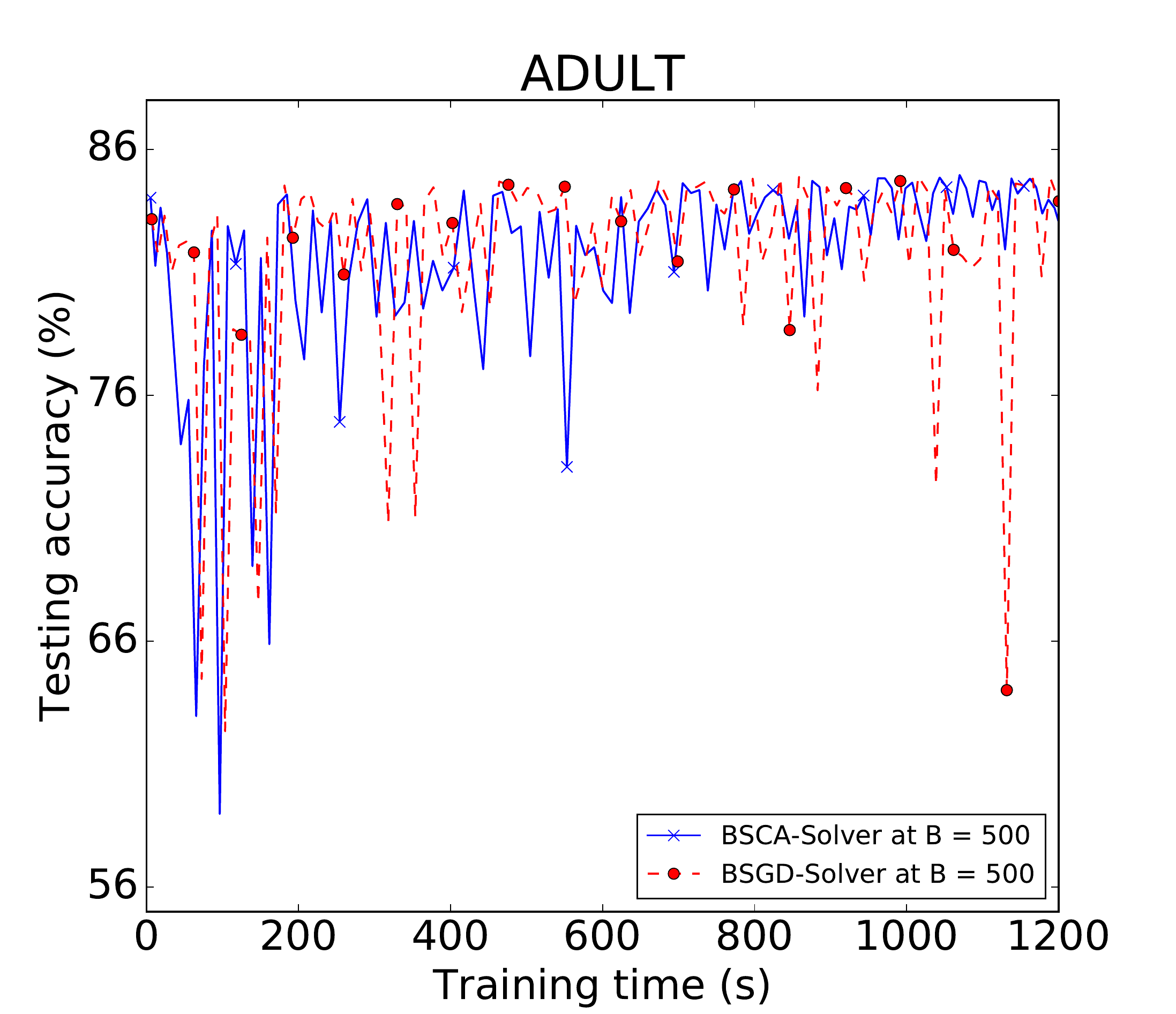}
\end{center}
\caption{
	Test accuracy for the primal and dual solvers at a budget of $B=500$.
	\label{figure:primaldualperf}
}
\end{figure}
Figure~\ref{figure:primaldualperf}
shows the corresponding evolution of the test error. In our experiment,
all budgeted models reach an accuracy that is nearly indistinguishable
from the exact SVM solution. The accuracy converges significantly faster
for the dual solver. For the primal solver we observe a long series of
peaks corresponding to significantly reduced performance. This
observation is in line with the observed optimization behavior. The
effect is particularly pronounced for the largest data sets SUSY and
COVERTYPE. More experimental data is found in the appendix,
including an investigation of the effect of the budget size.

\paragraph{Convergence Behavior}
The next experiment validates the predictions of Lemma~\ref{lemma:rate}
when using a budget. Figure~\ref{figure:mergingFraction}
displays the fraction of merging steps for different budget sizes
applied to the dual and primal solvers. We find the predictions of the
lemma being approximately valid also when using a budget. The figure
highlights an interesting difference in the optimization behavior
between BSGD and BSCA: while the former makes non-zero steps on all
support vectors (training points with a margin of at most one), the
latter manages to fix the dual variables of margin violators (training
points with a margin strictly less than one) at the upper bound~$C$.
\begin{figure}
	\begin{center}
		\includegraphics[width=0.39\textwidth]{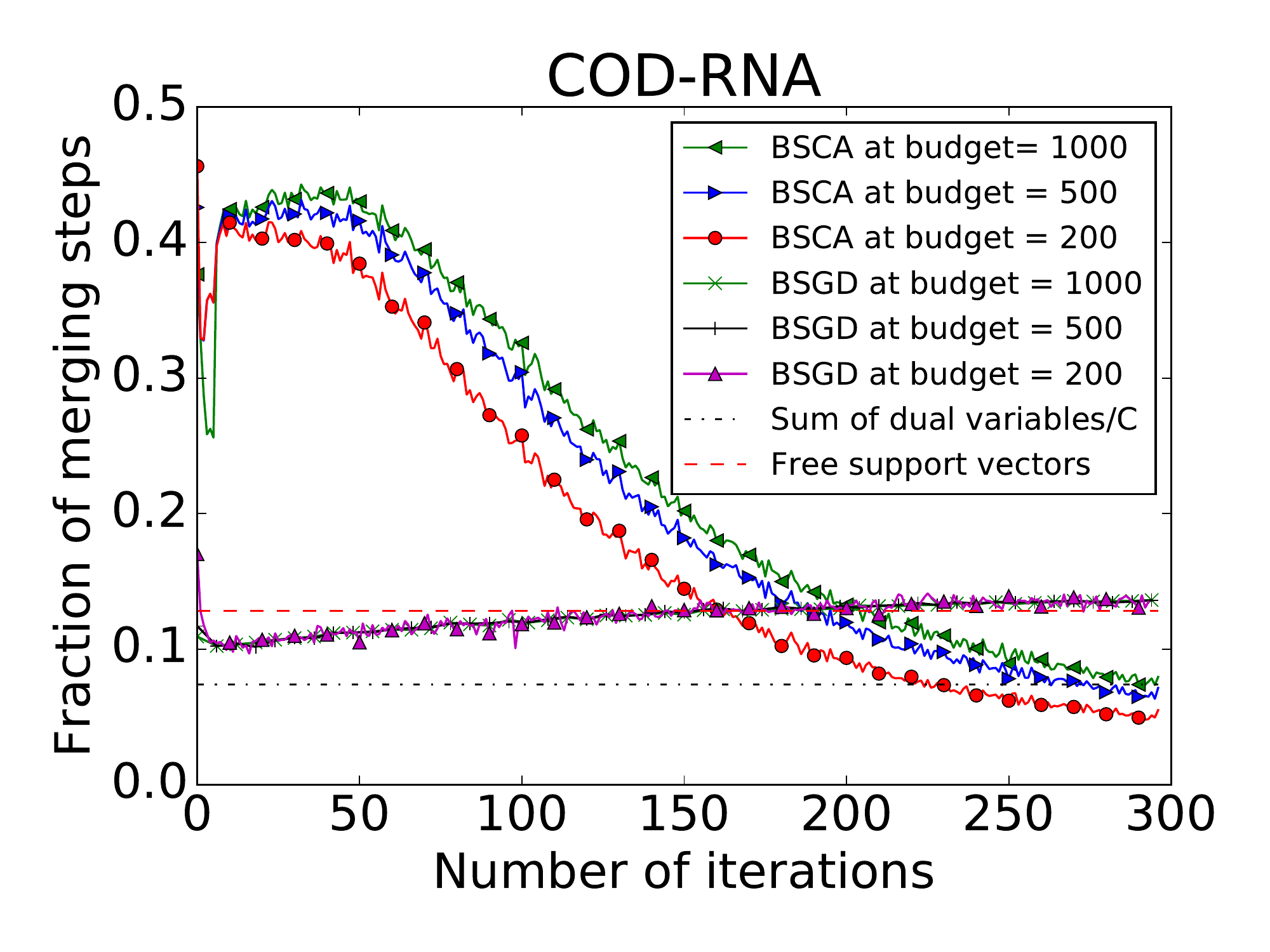}
	\end{center}
	\caption{Fraction of merging steps over a large number of epochs at
		budgets $B \in \{200, 500, 1000\}$. Corresponding curves for the
		other data sets are provided in the appendix.
		\label{figure:mergingFraction}
	}
\end{figure}

\paragraph{Discussion}

Our results do not only indicate that the optimization behavior of BSGD
and BSCA is significantly different, they also demonstrate the
superiority of the dual approach for SVM training, in terms of
optimization behavior as well as in terms of test accuracy. We attribute
its success to the good fit of coordinate descent to the box-constrained
dual problem, while the primal solver effectively ignores the upper
bound (which is not represented explicitly in the primal problem),
resulting in large oscillations. Importantly, the improved optimization
behavior translates directly into better learning behavior. The
introduction of budget maintenance techniques into the dual solver does
not change this overall picture, and hence yields a viable route for
fast training of kernel machines.

\section{Concluding Remarks}

We have presented the first dual decomposition algorithm for support
vector machine training honoring a budget, i.e., an upper bound on the
number of support vectors. This approximate SVM training algorithm
combines fast iterations enabled by the budget approach with the fast
convergence of a dual decomposition algorithm. Like its primal cousin,
it is fundamentally limited only by the approximation error induced by
the budget. We demonstrate significant speed-ups over primal budgeted
training, as well as increased stability. Overall, for training SVMs
on a budget, we can clearly recommend our method as a plug-in
replacement for primal methods. It is rather straightforward to extend
our algorithm to other kernel machines with box-constrained dual
problems.

\bibliographystyle{plain}

\section*{Appendix A: Data Sets and Hyperparameters}

The test problems were selected according to the following criteria,
which taken together imply that applying the budget method is a
reasonable choice:
\begin{itemize}
\item
	The feature dimension is not too large. Therefore a linear SVM
	performs rather poorly compared to a kernel machine.
\item
	The problem size is not too small. The range of sizes spans more
	than two orders of magnitude.
\end{itemize}
The hyperparameters $C$ and $\gamma$ were $\log_2$ encoded and tuned on
an integer grid with cross-validation using LIBSVM, i.e., aiming for the
best possible performance of the non-budgeted machine. The budget was
set to $B=500$ in all experiments, unless stated otherwise. This value
turns out to offer a reasonable compromise between speed and accuracy on
all problems under study.

\section*{Appendix B: Impact of the Budget}

In this section we investigate the impact of the budget and its size on
optimization and learning behavior. We start with an experiment
comparing primal and dual solver without budget. The results are
presented in figure~\ref{figure:bsca_bsgd_0}. It is apparent that the
principal differences between BSGD and BSCA remain the same when run
without budget constraint, i.e., the most significant differences stem
from the quite different optimization behavior of stochastic gradient
descent and stochastic coordinate ascent. The SGD learning curves are
quite noisy with many downwards peaks. The results are in line with
experiments on linear SVMs by \cite{Fan:2008,Hsieh:2008}.

To investigate the effect of the budget size,
Figure~\ref{figure:bsca_bsgd_200} provides test accuracy curves for a
reduced budget size of $B=200$. For some of the test problems this
budget it already rather small, resulting in sub-optimal learning
behavior. Generally speaking, BSCA clearly outperforms BSGD. However,
BSCA fails on the IJCNN data set, while BSGD fails to deliver high
quality solutions on SUSY and COVERTYPE.

Figure \ref{figure:bsca_bsgd_te_200_500} aggregates the data in a
different way, comparing the test accuracy achieved with different
budgets on a common time axis. In this presentation it is easy to read
off the speed-up achievable with a smaller budget.
Unsurprisingly, BSCA with budget $B=200$ is much faster than the same
algorithm with budget $B=500$ when run for the same number of epochs.
However, when it comes to achieving a good test error quickly, the
results are mixed. While the small budget apparently suffices on
COVERTYPE and SUSY, the provided number of epochs does not
suffice to reach good results on IJCNN, where the solver with $B=500$ is
significantly faster. Figure~\ref{figure:bsca_bsgd_objfn_200_500}
presents a similar analysis, but with primal and dual objective function.
Overall it underpins the learning (test accuracy) results, however, it
also reveals a drift effect of the dual solver in particular for the
smaller budget $B=200$, with both objectives rising. This can happen if
the weight degradation becomes large and the gradient computed based on
the budgeted representation does not properly reflect the dual gradient
any more.

\begin{figure}[h]
\begin{center}
	\includegraphics[width=0.32\columnwidth]{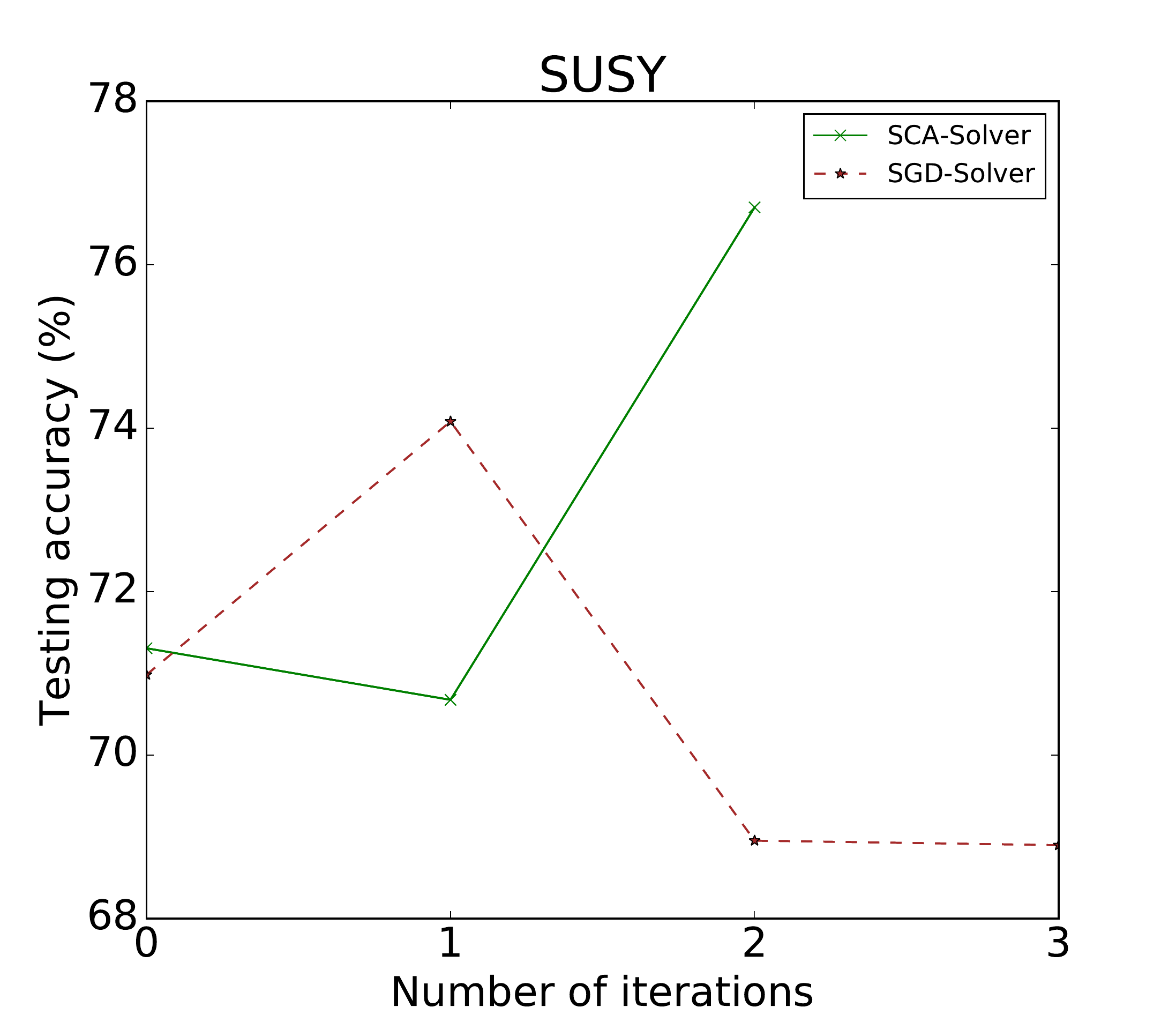}
	\includegraphics[width=0.32\columnwidth]{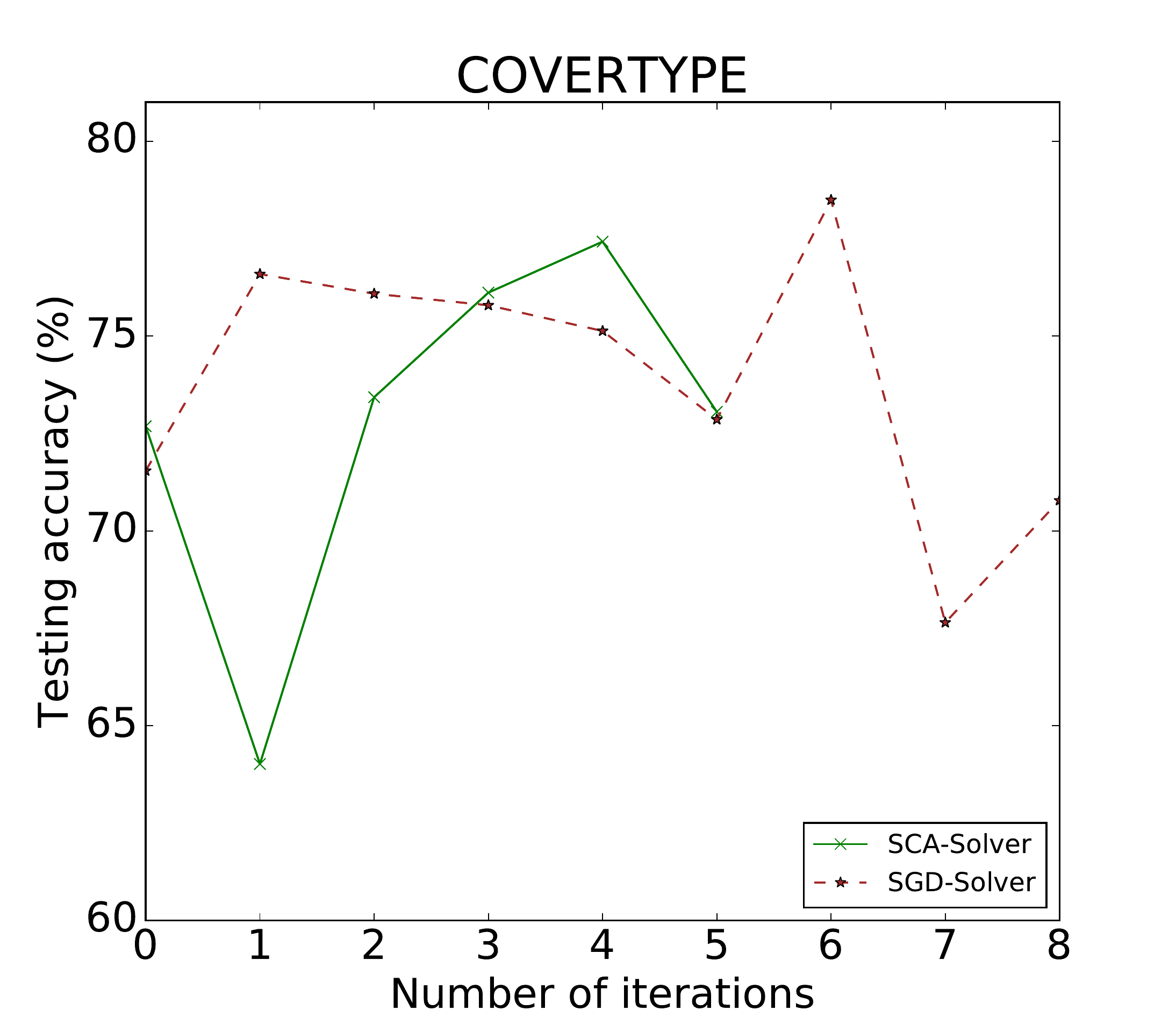}
	\includegraphics[width=0.32\columnwidth]{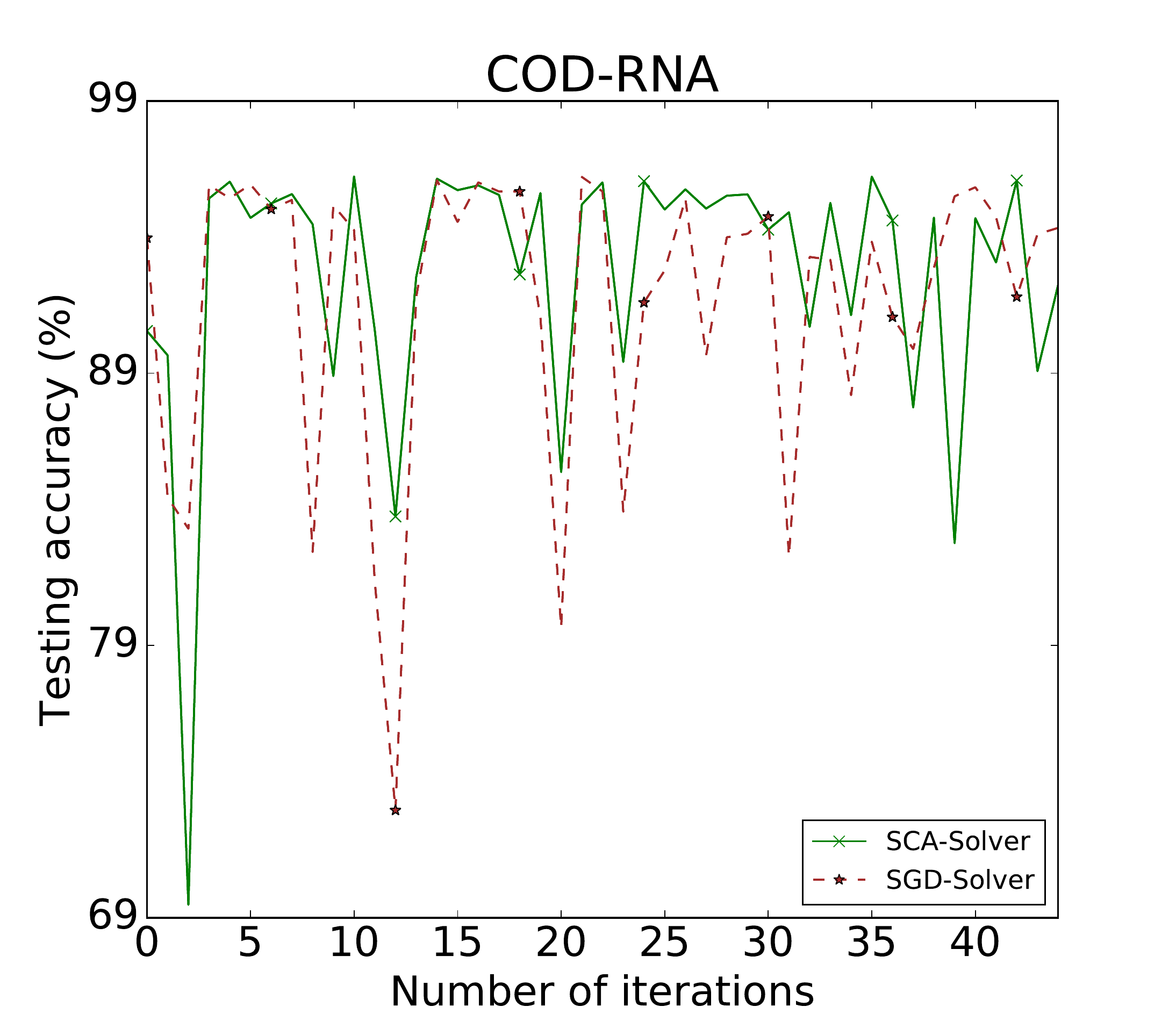}
	\includegraphics[width=0.32\columnwidth]{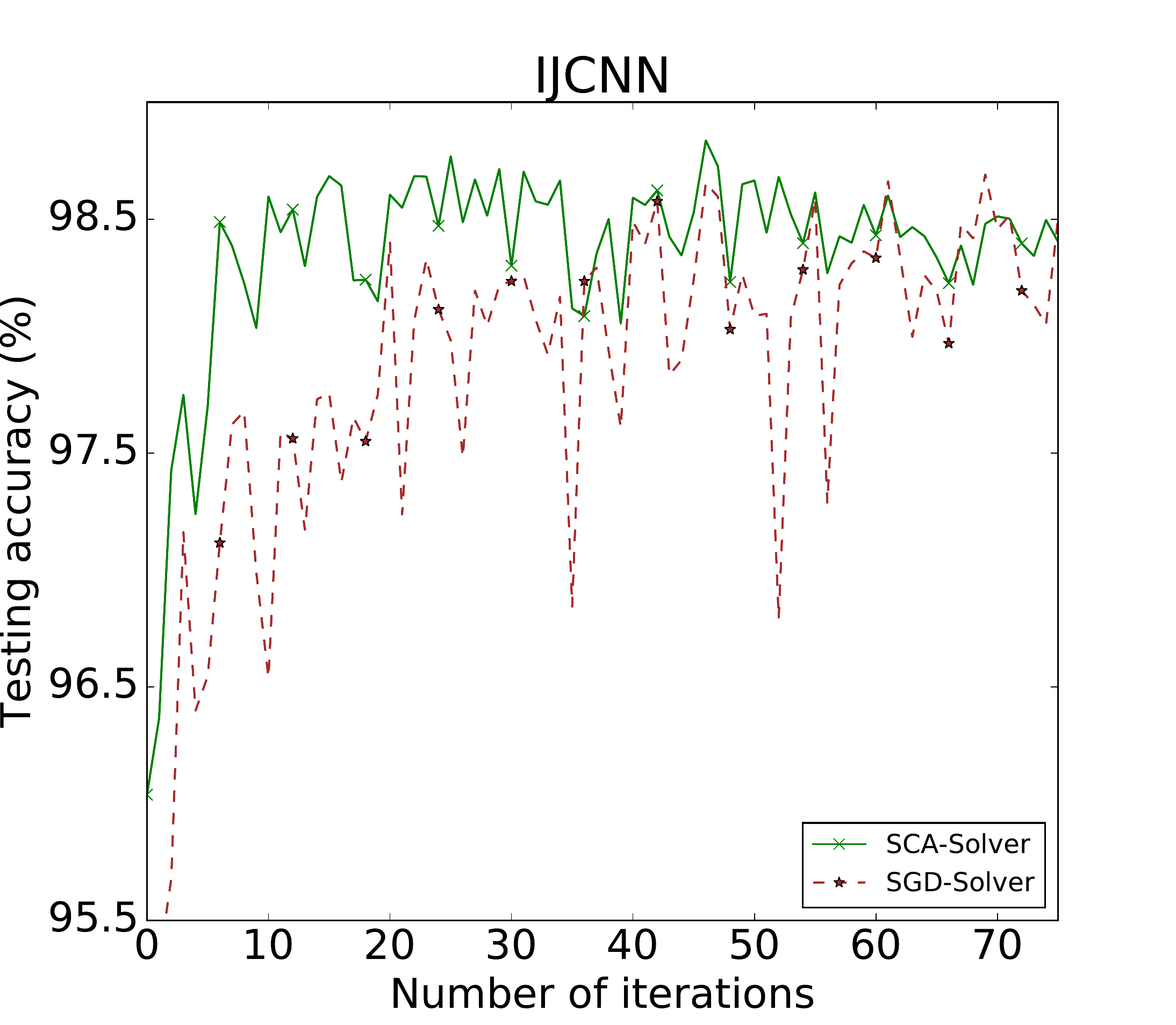}
	\includegraphics[width=0.32\columnwidth]{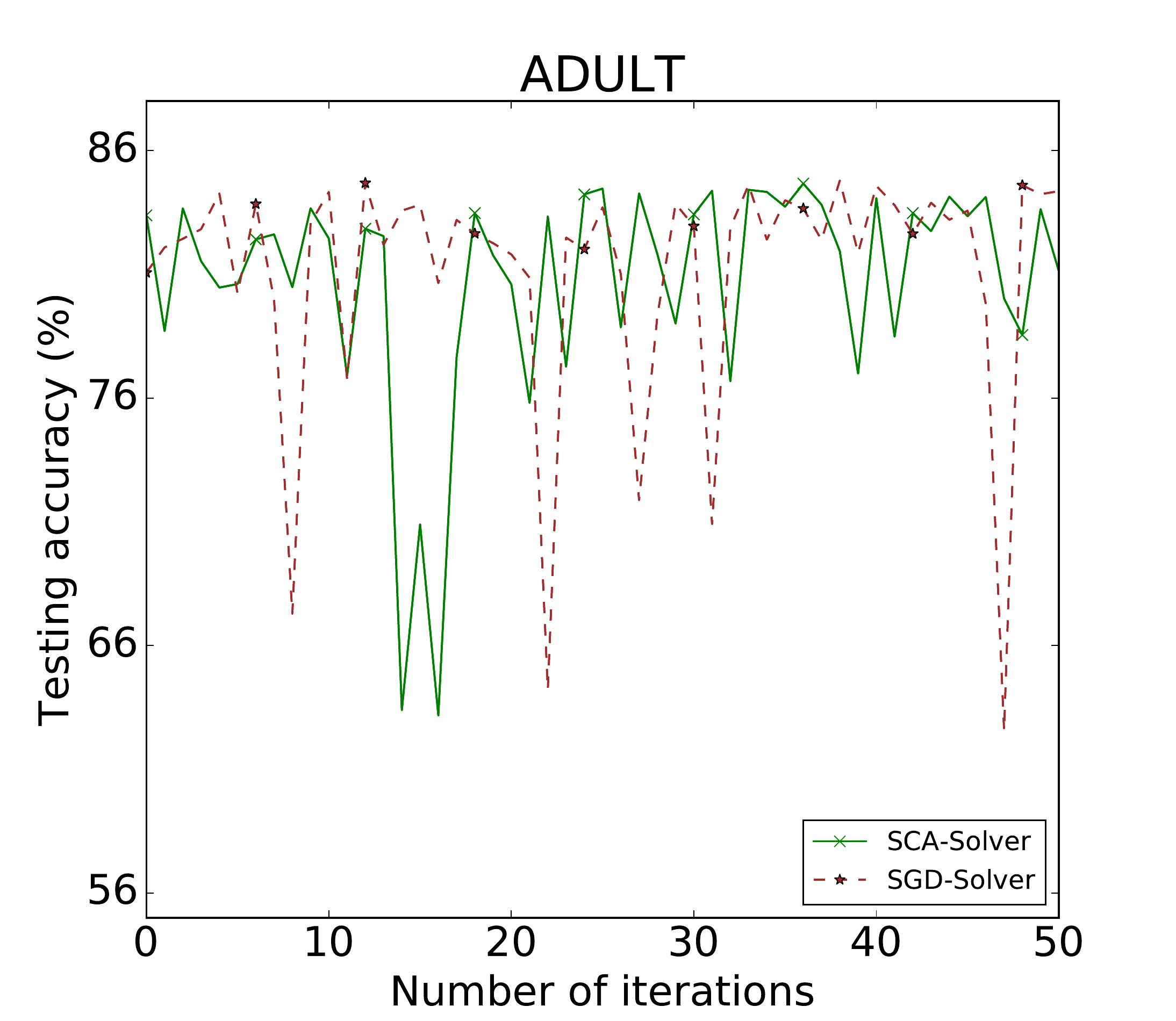}
\end{center}
\caption{
	Test accuracy results of solvers based directly on SCA and SGD, without budget.
	The results are monitored every $300,000$ iterations for SUSY and COVERTYPE,
	and every $10,000$ iterations for all other data sets.
	\label{figure:bsca_bsgd_0}
}
\end{figure}

\begin{figure}[h]
\begin{center}
	\includegraphics[width=0.32\columnwidth]{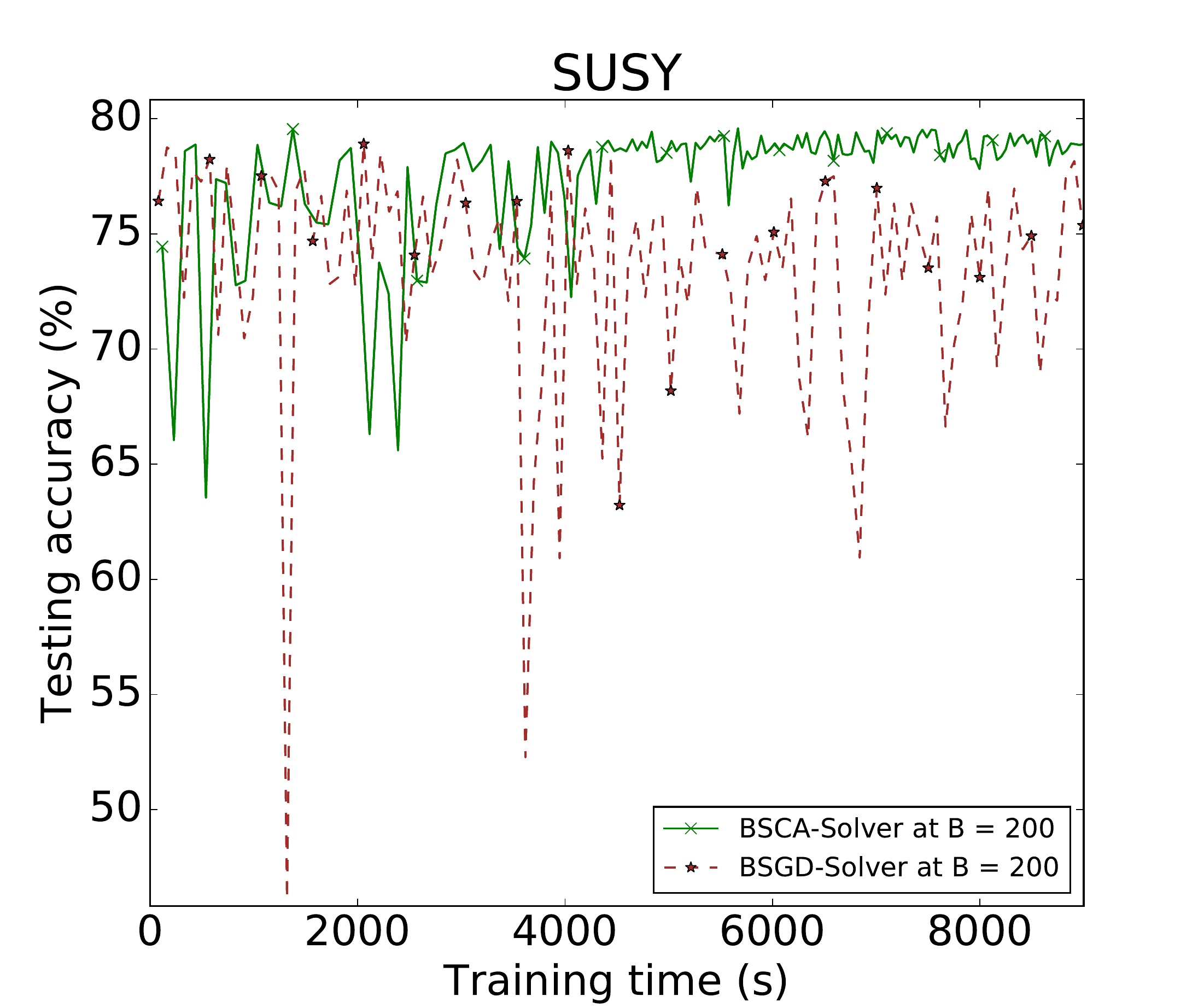}
	\includegraphics[width=0.32\columnwidth]{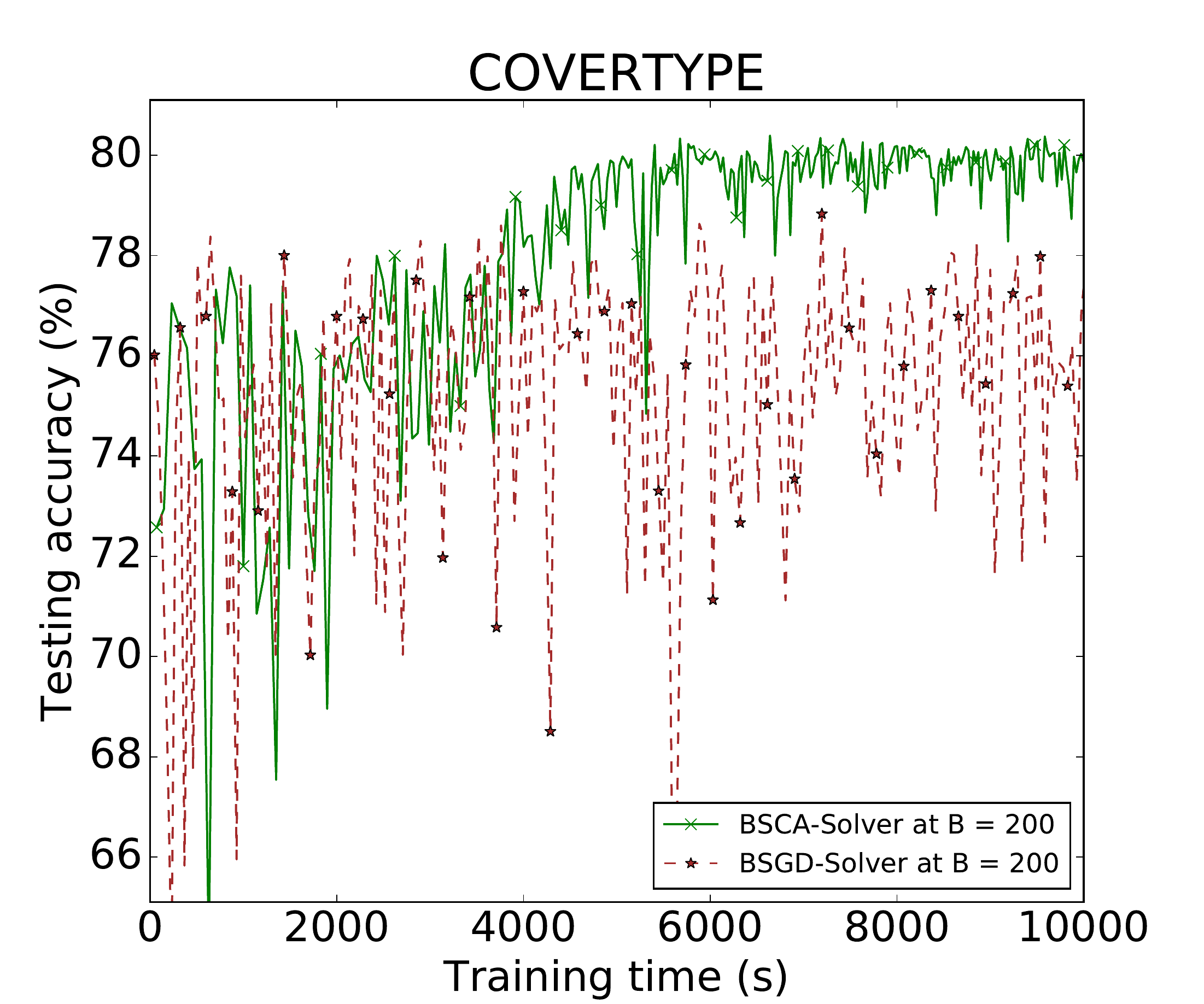}
	\includegraphics[width=0.32\columnwidth]{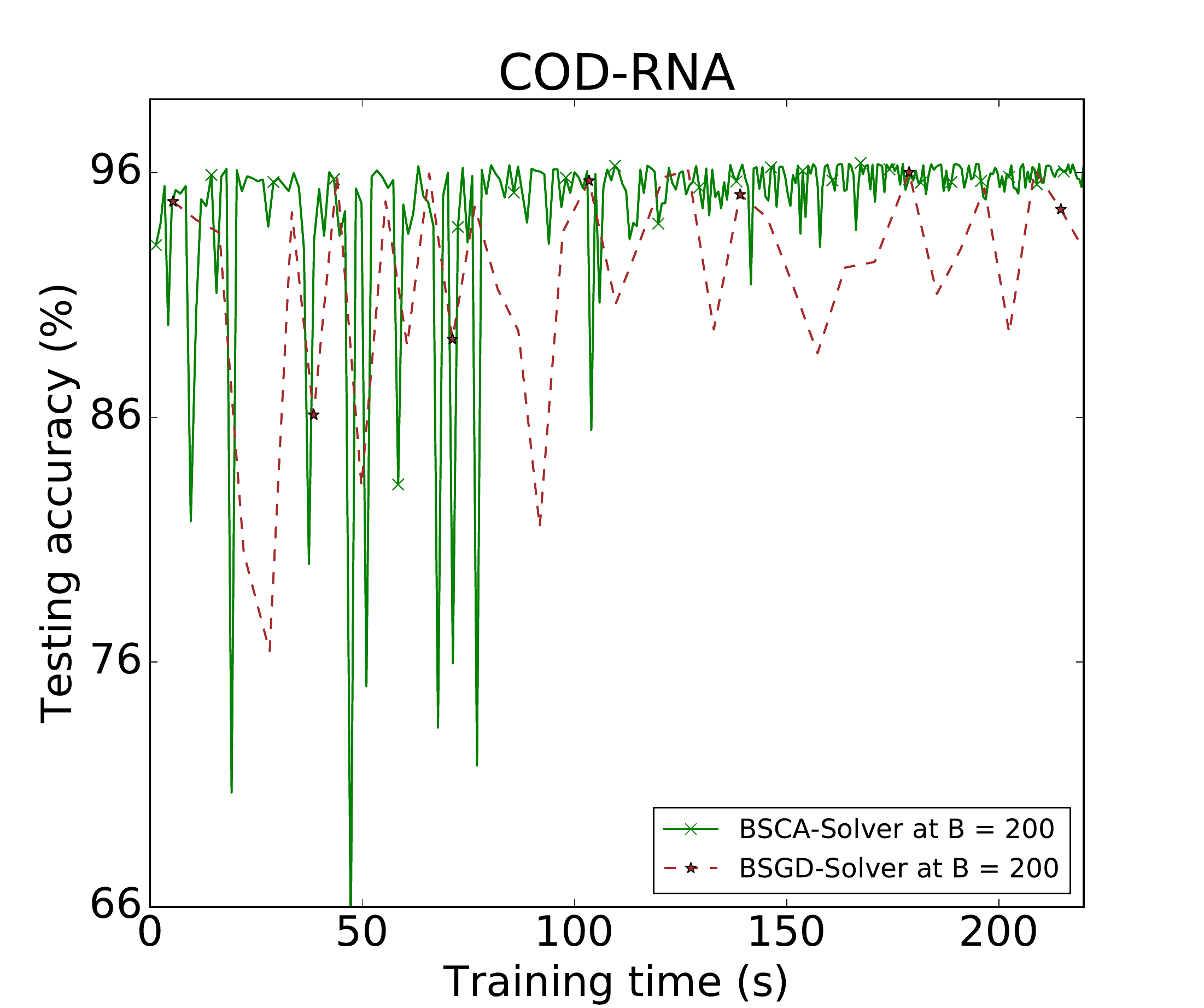}
	\includegraphics[width=0.32\columnwidth]{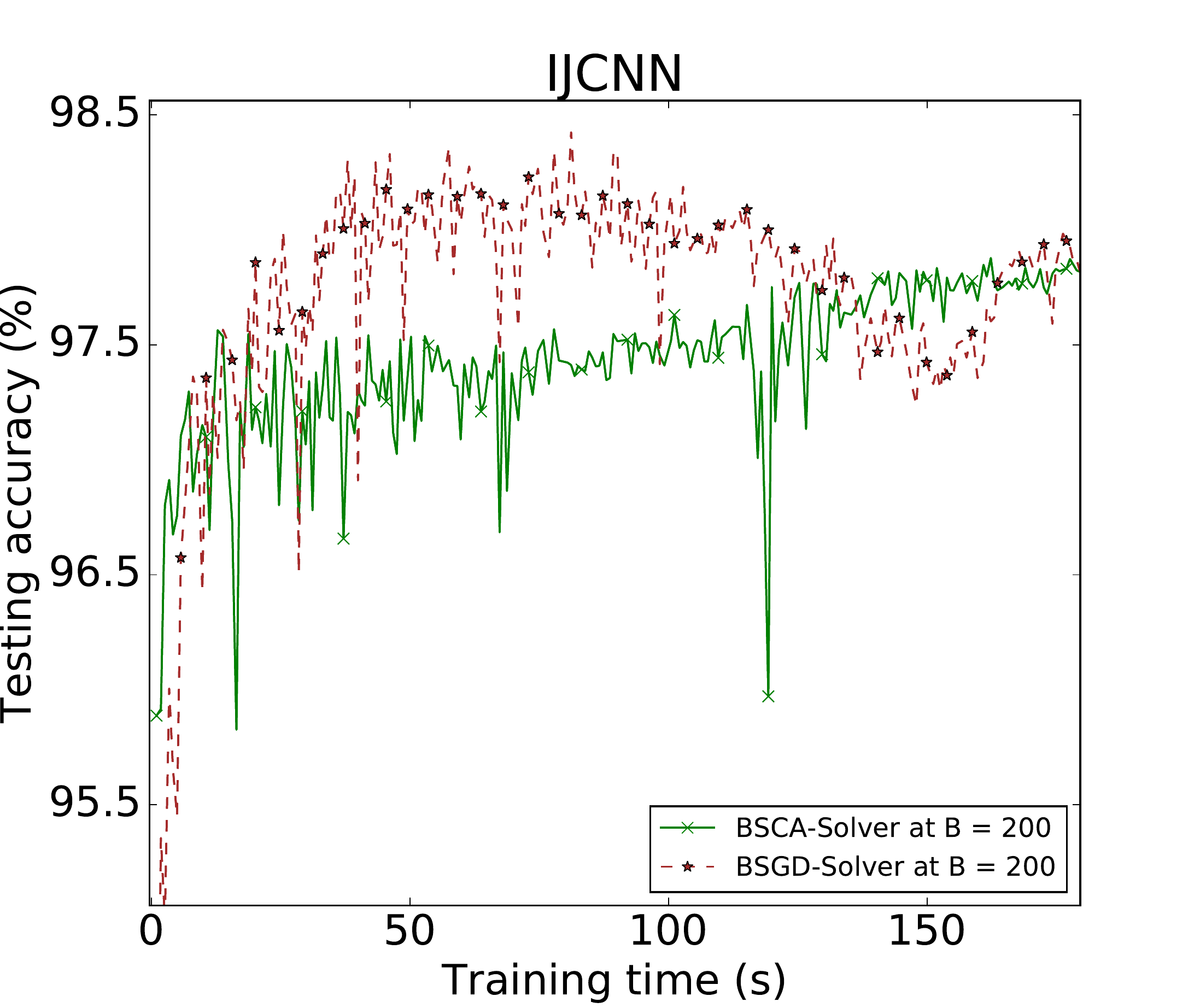}
	\includegraphics[width=0.32\columnwidth]{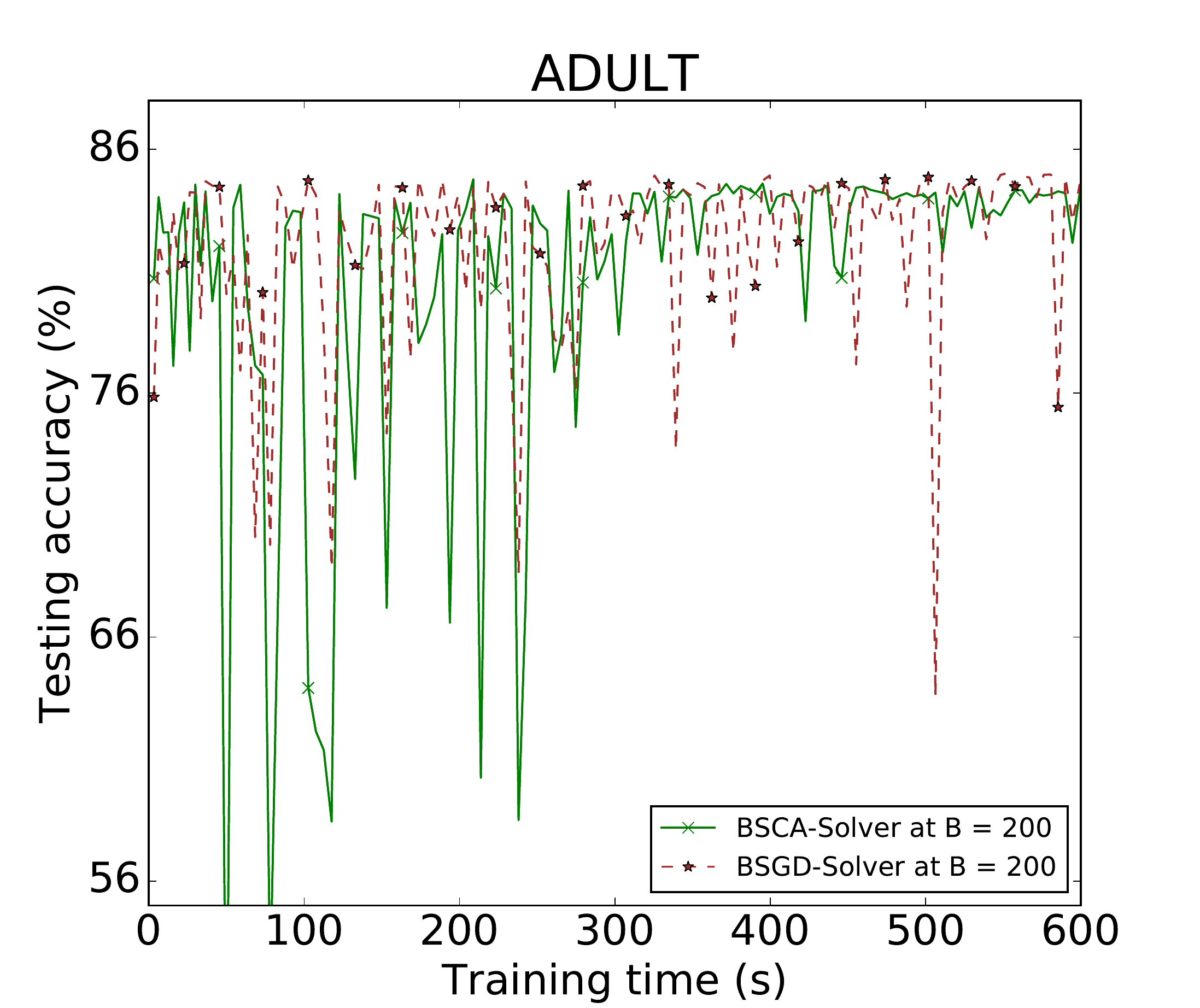}
\end{center}
\caption{
	Test accuracy results for BSCA and BSGD at a budget of $200$.
	\label{figure:bsca_bsgd_200}
}
\end{figure}
\begin{figure}[h]
\begin{center}
	\includegraphics[width=0.32\columnwidth]{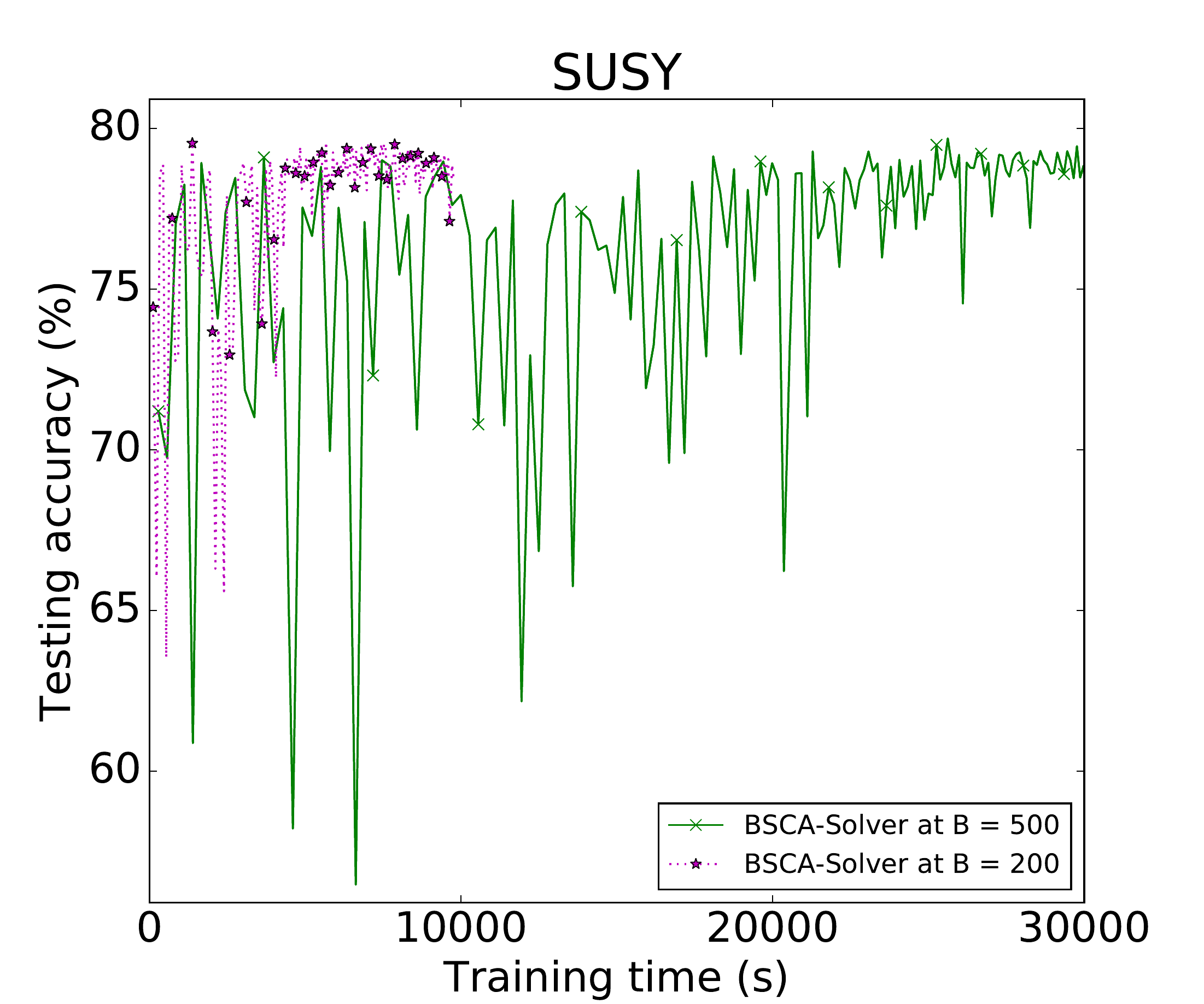}
	\includegraphics[width=0.32\columnwidth]{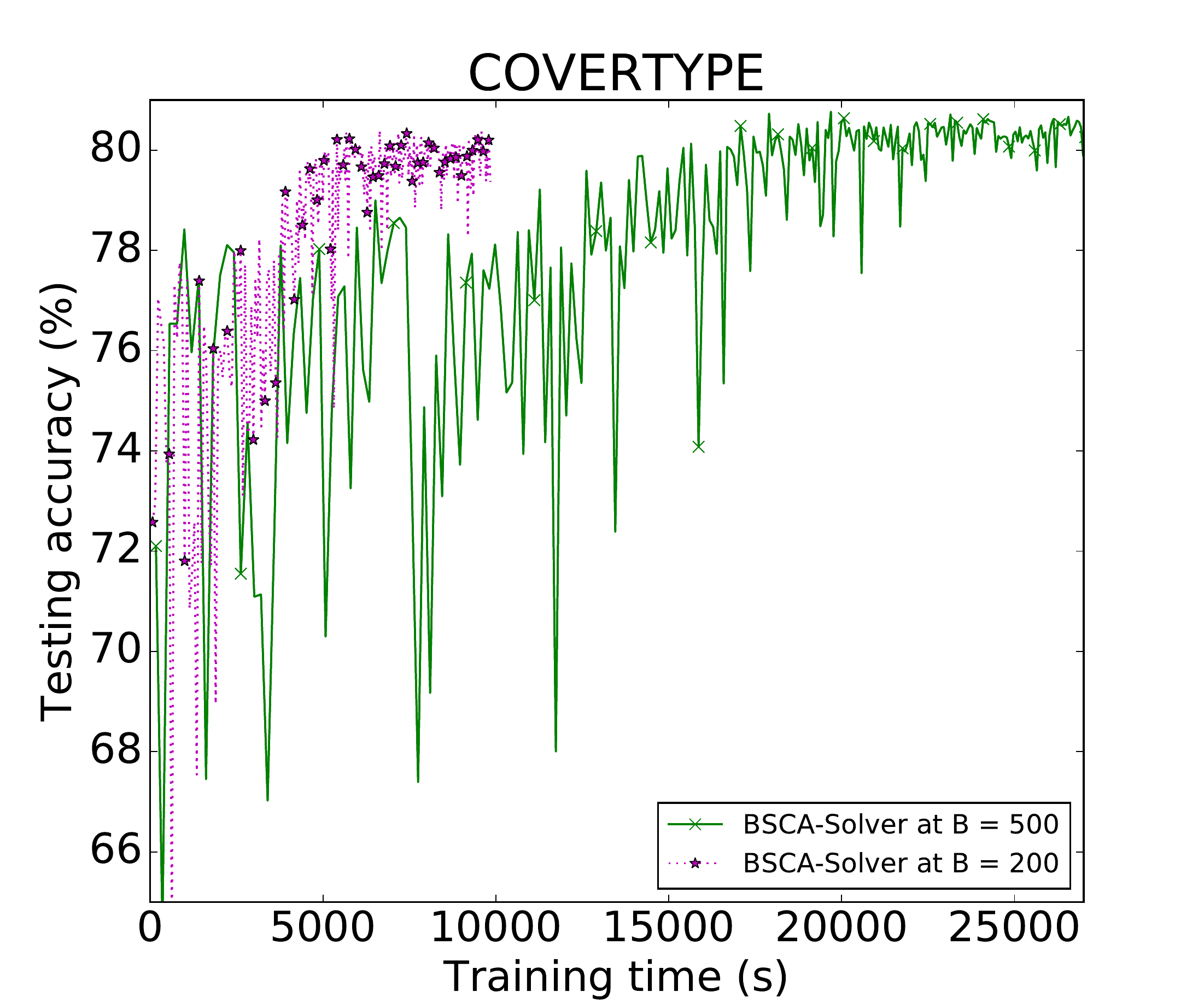}
	\includegraphics[width=0.32\columnwidth]{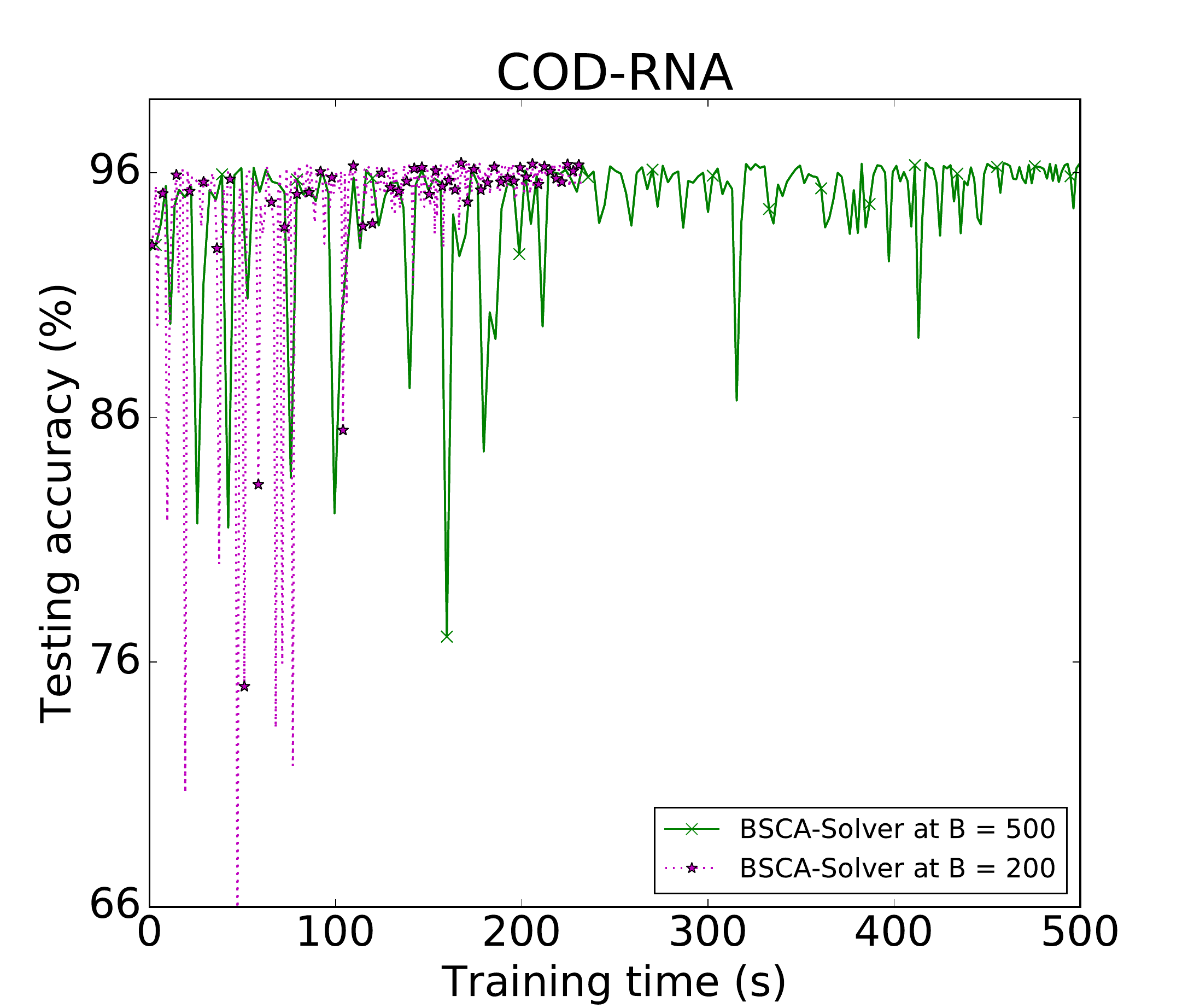}
	\includegraphics[width=0.32\columnwidth]{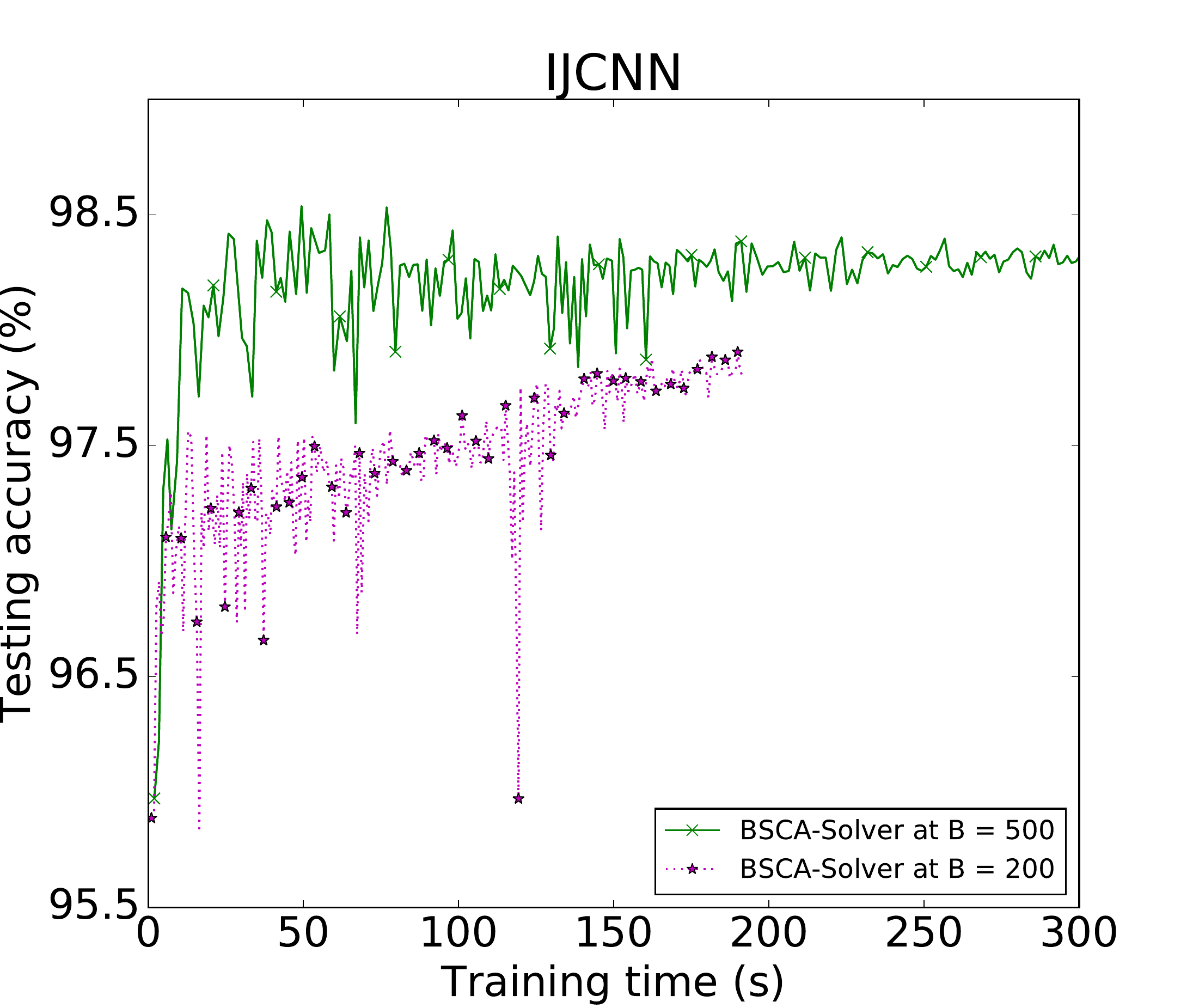}
	\includegraphics[width=0.32\columnwidth]{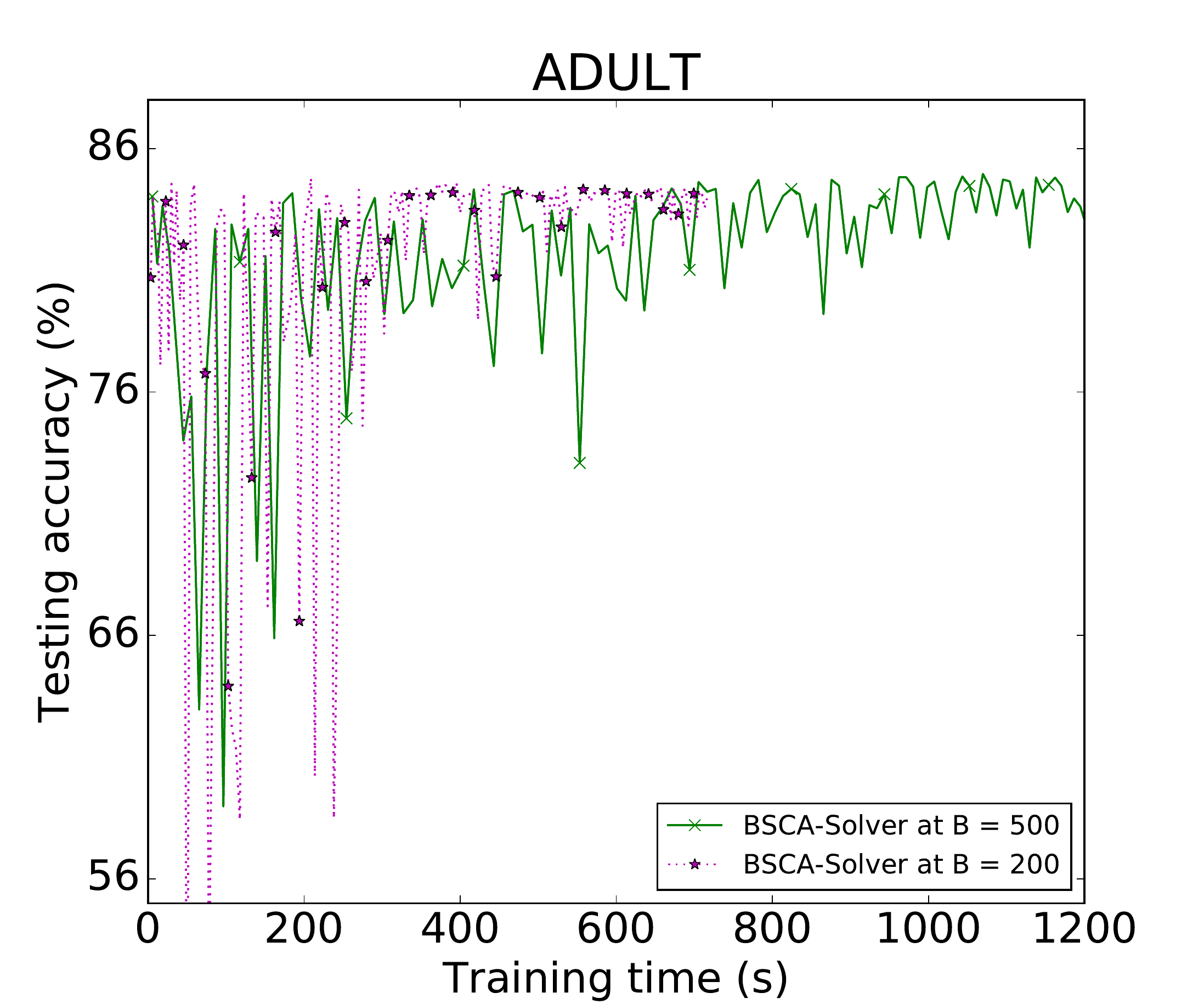}
\end{center}
\caption{
	Test accuracy over time for BSCA at budgets $B \in \{200, 500\}$.
	\label{figure:bsca_bsgd_te_200_500}
}
\end{figure}
\begin{figure}[h]
\begin{center}
	\includegraphics[width=0.32\columnwidth]{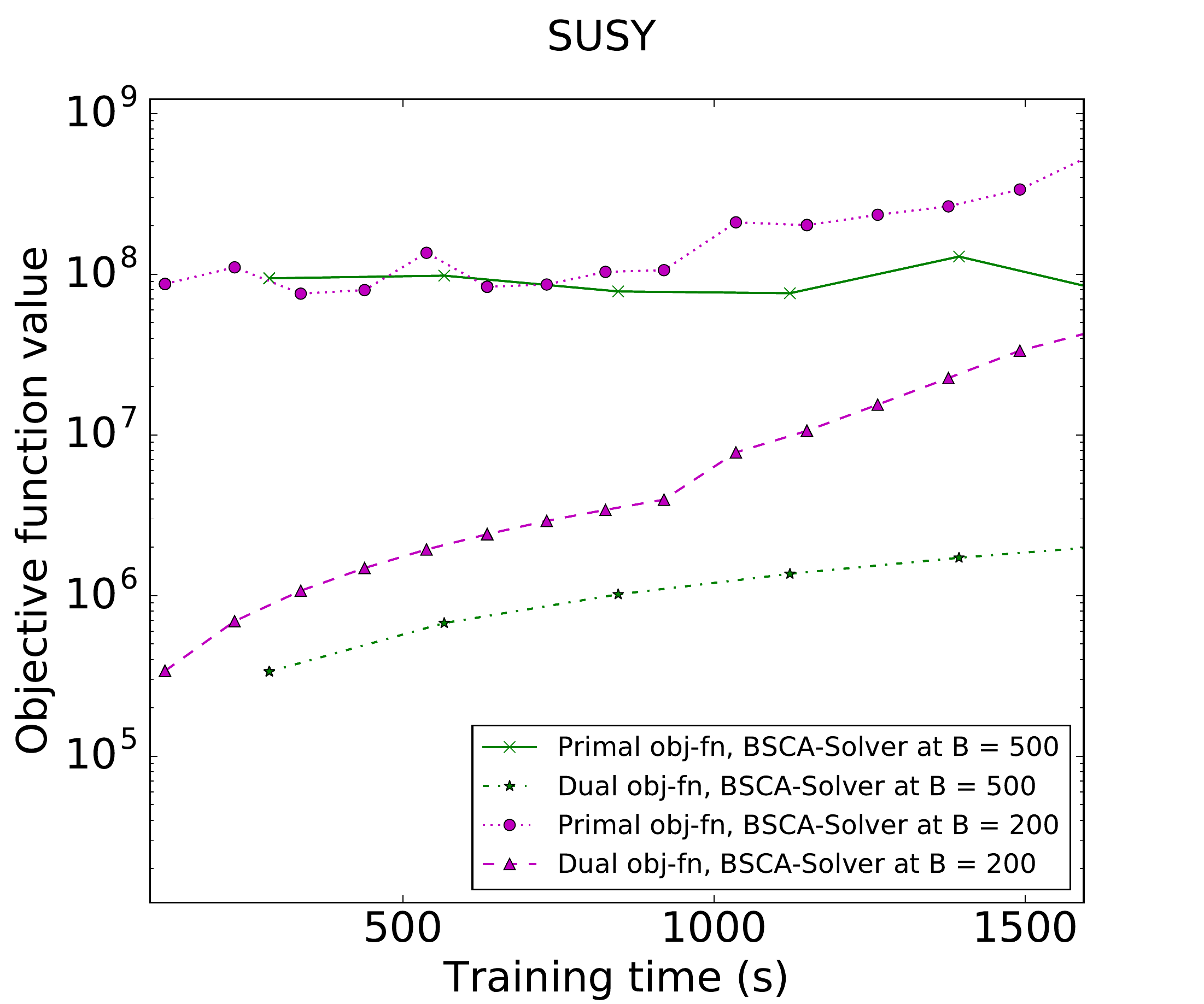}
	\includegraphics[width=0.32\columnwidth]{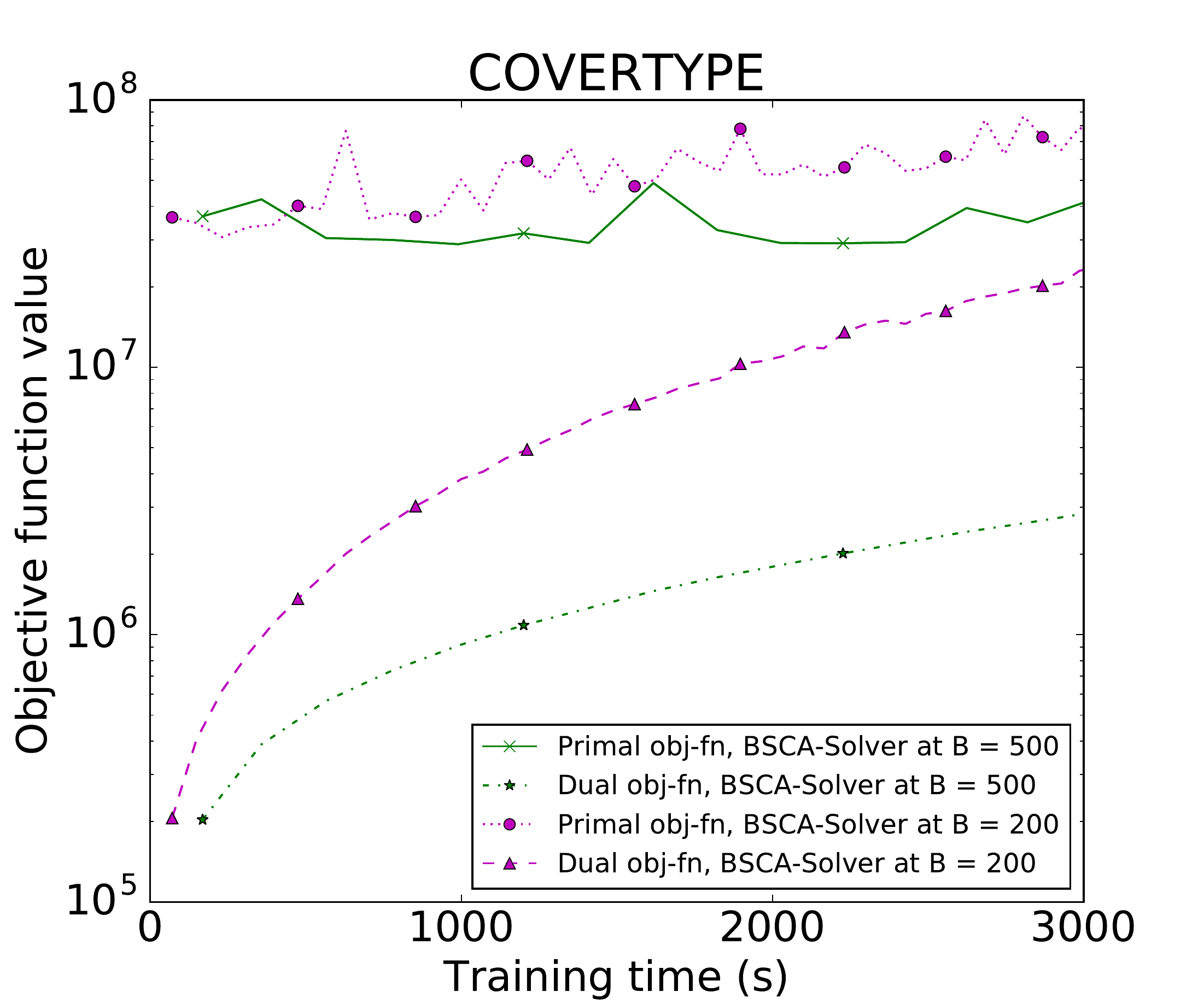}
	\includegraphics[width=0.32\columnwidth]{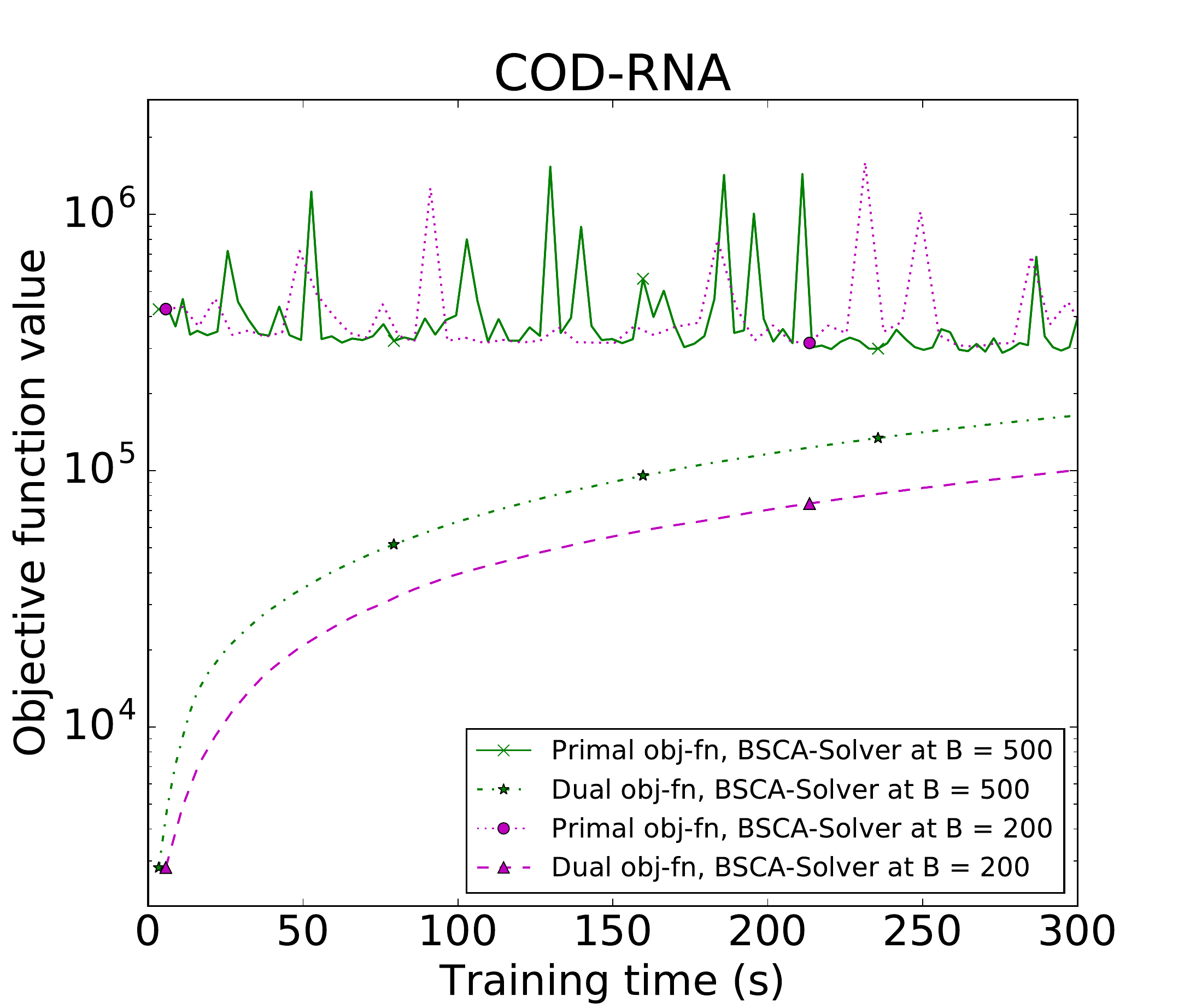}
	\includegraphics[width=0.32\columnwidth]{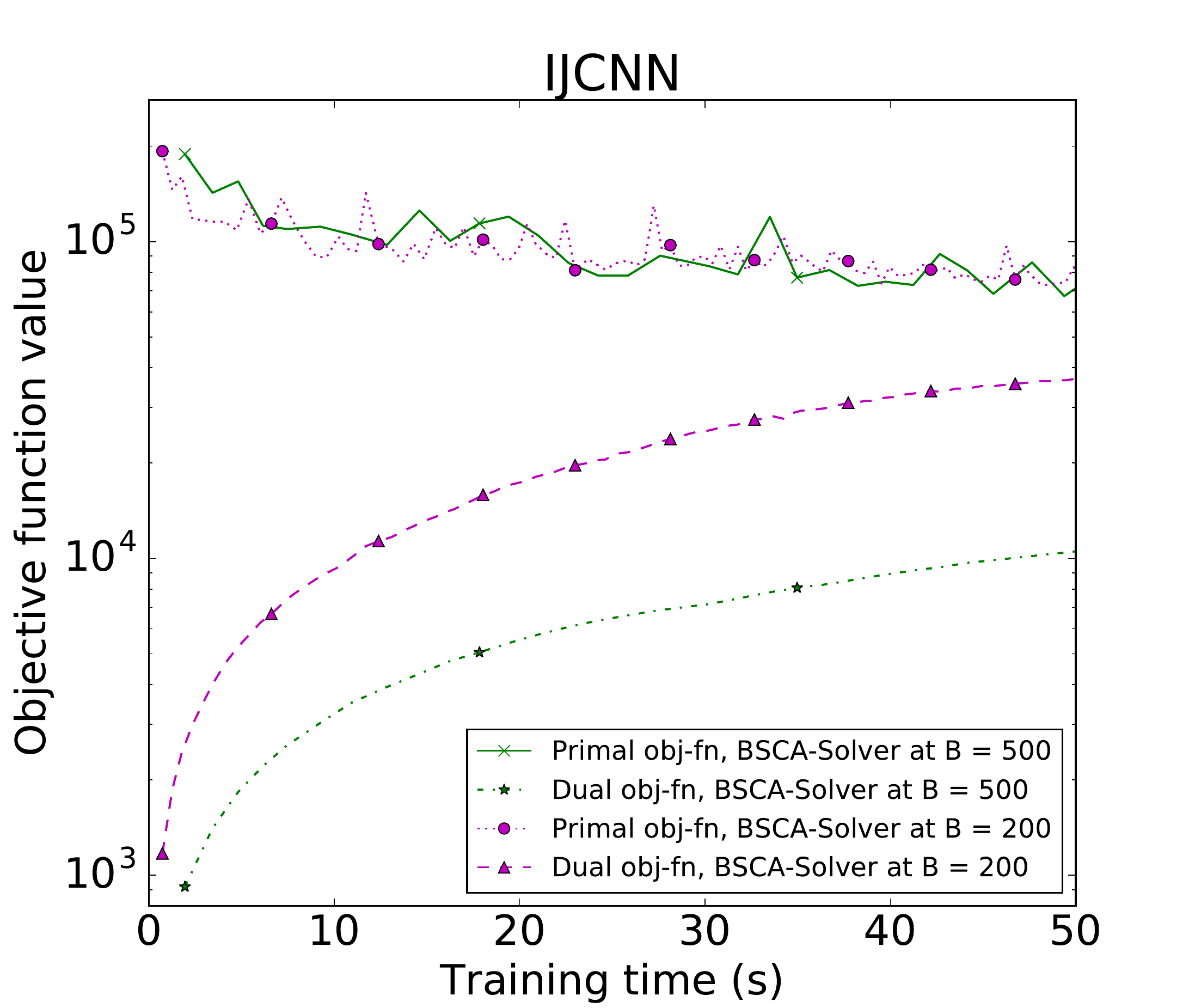}
	\includegraphics[width=0.32\columnwidth]{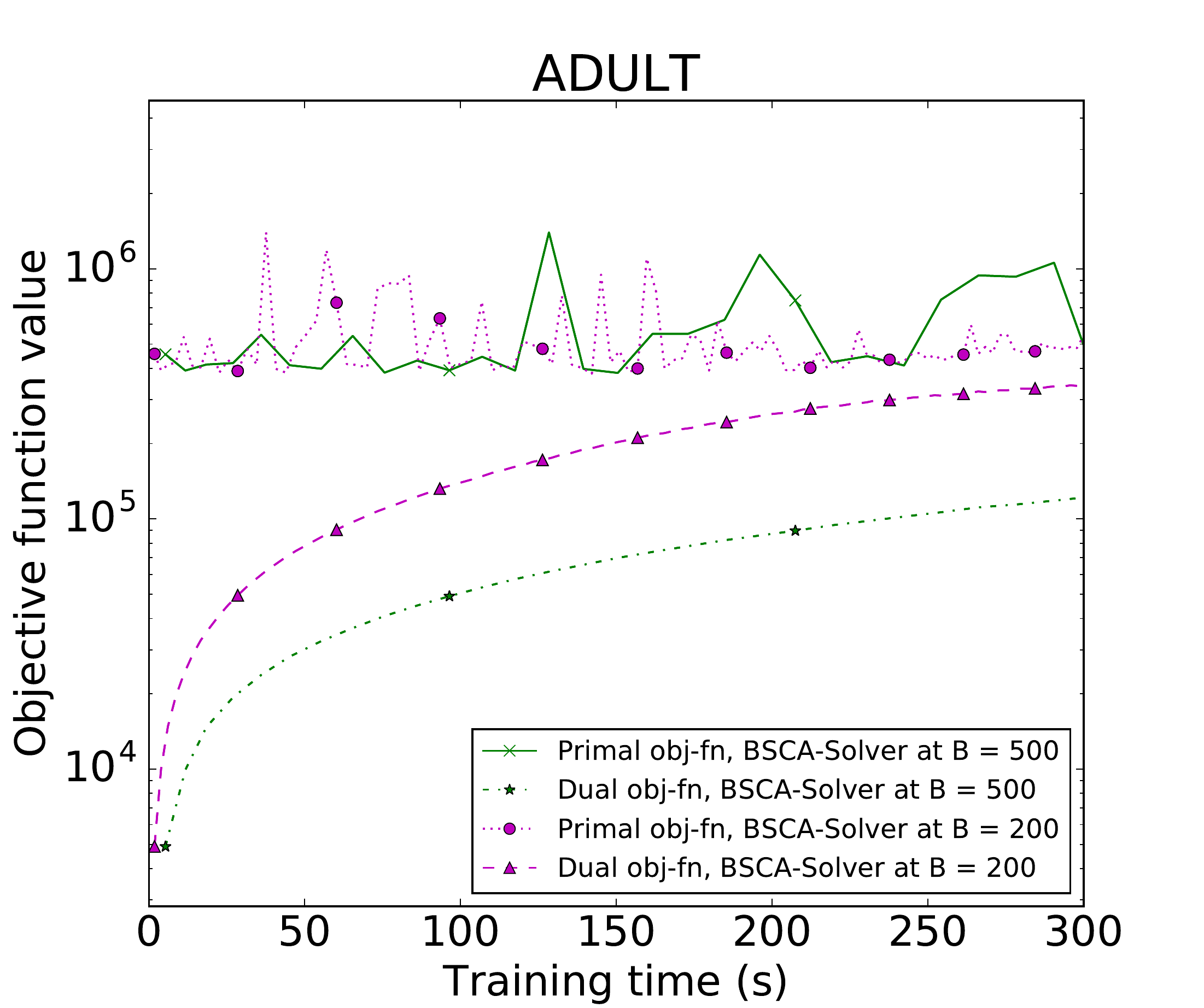}	
\end{center}
\caption{
	Primal and dual objective function over time for BSCA at budgets $B \in \{200, 500\}$.
	\label{figure:bsca_bsgd_objfn_200_500}
}
\end{figure}

\section{Appendix C: Merging Steps}

Figure~\ref{figure:mergingFraction} in the main paper shows the fraction
of merging steps for the COD-RNA problem.
Figure~\ref{figure:mergingFractionSup} provides the same data for the
remaining data sets, with very similar results.
\begin{figure}[hbt!]
\begin{center}
	\includegraphics[width=0.45\columnwidth]{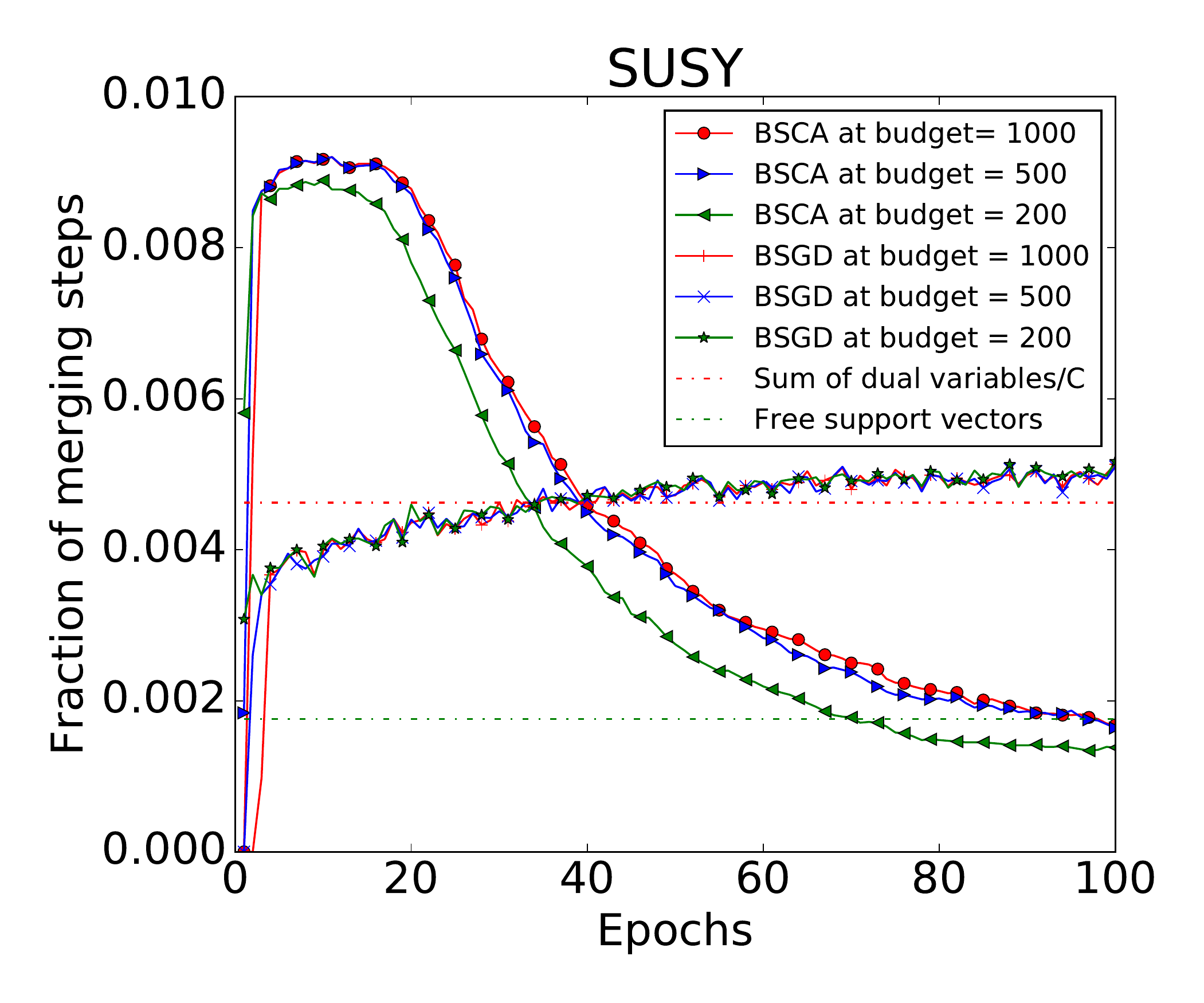}
	\includegraphics[width=0.45\columnwidth]{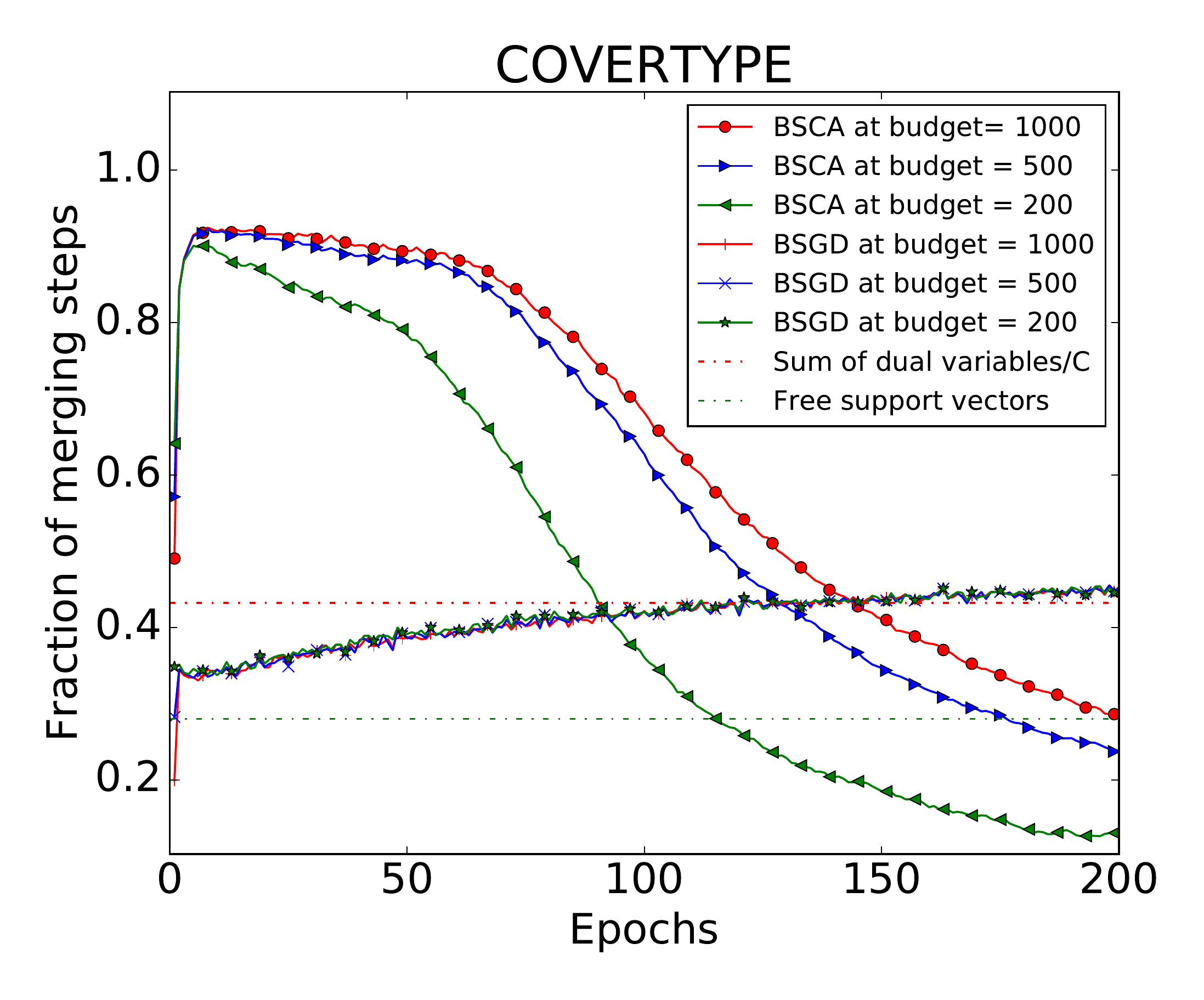}
	\includegraphics[width=0.45\columnwidth]{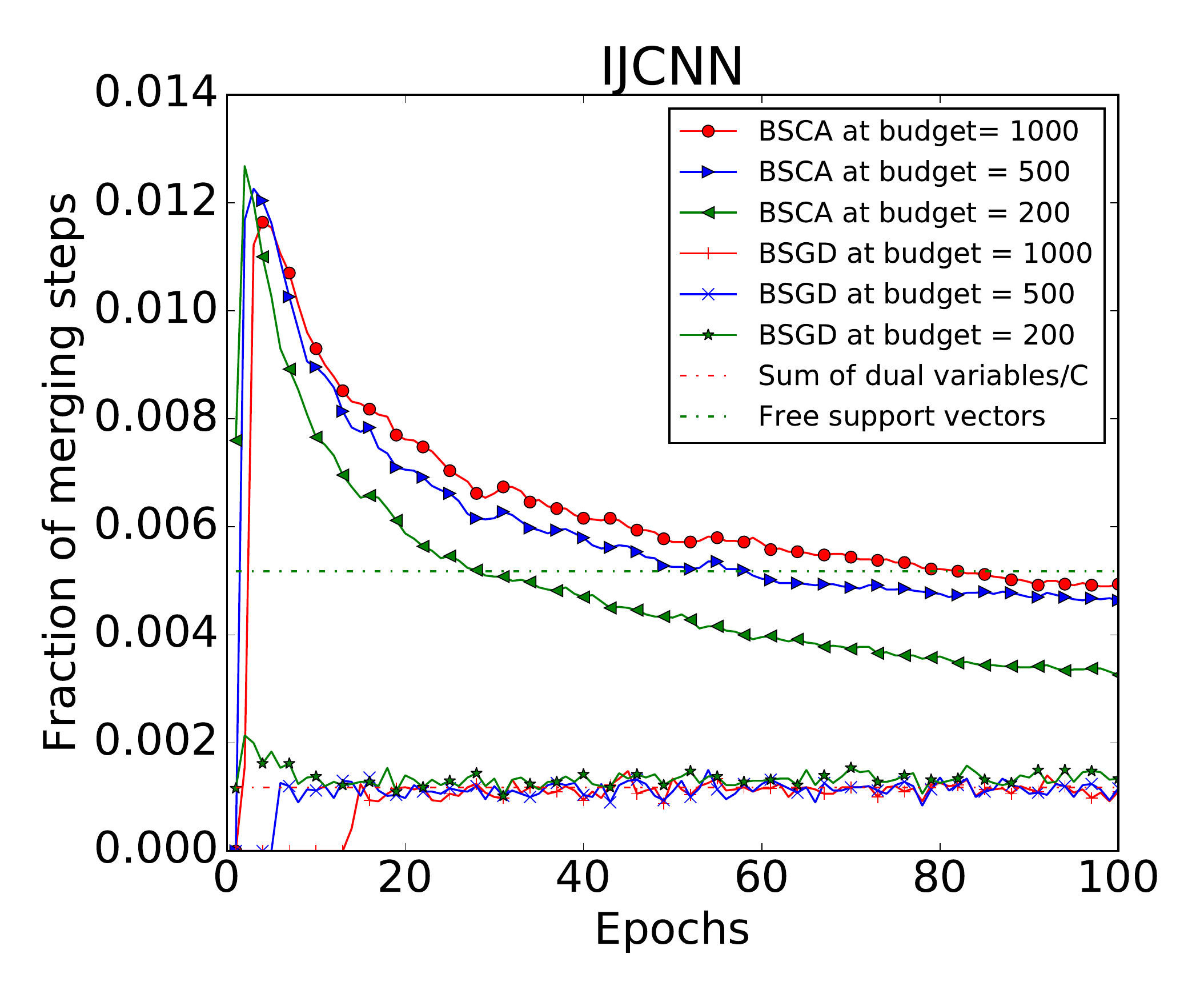}
	\includegraphics[width=0.45\columnwidth]{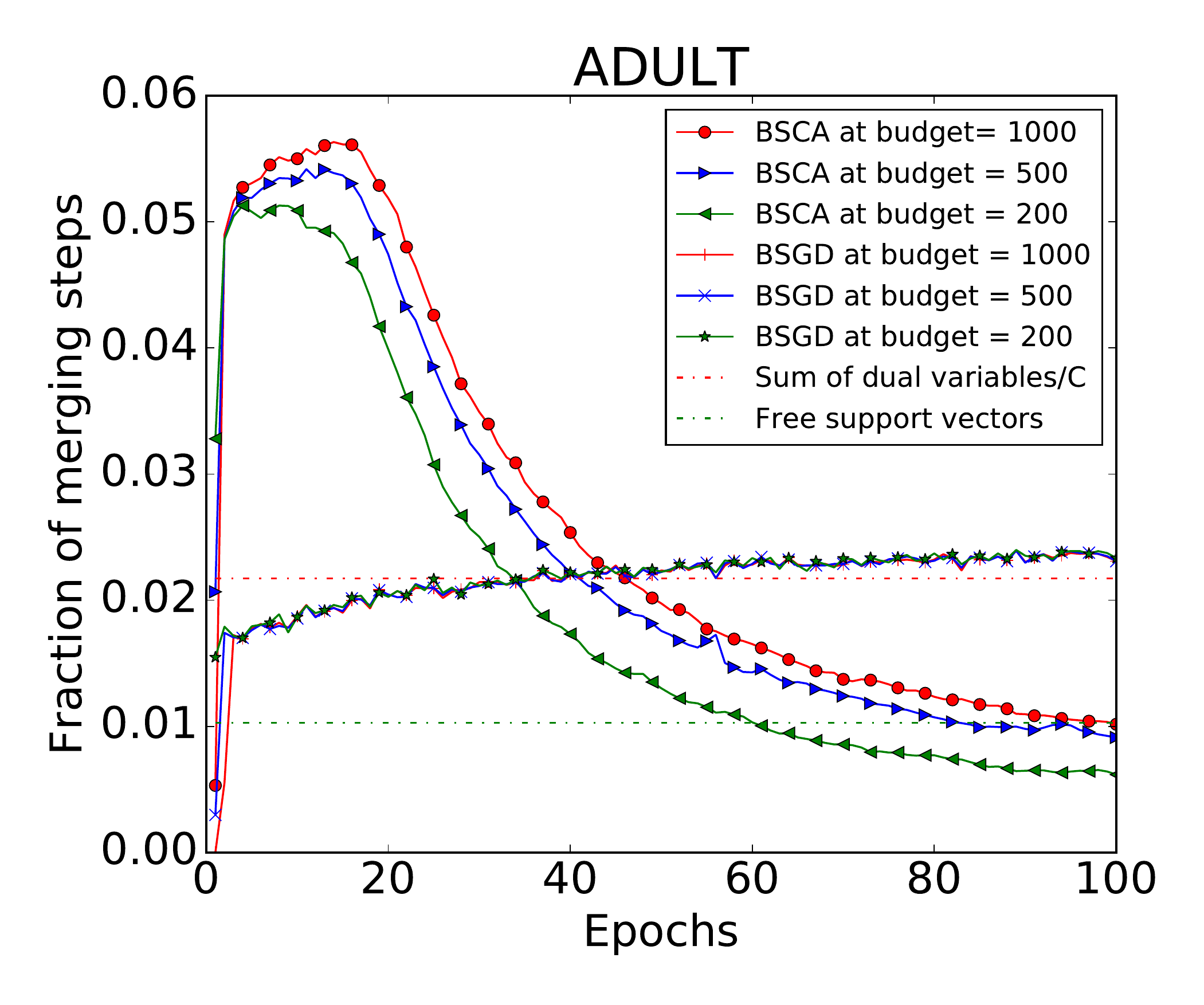}
\end{center}
\caption{
	Fraction of merging steps over a large number of epochs at budgets
	$B \in \{200, 500, 1000\}$.
	\label{figure:mergingFractionSup}
}
\end{figure}

\end{document}